\documentclass[11pt,letterpaper]{article}

\usepackage[margin=1in]{geometry}
\usepackage[T1]{fontenc}
\usepackage{lmodern,microtype}
\usepackage{aliascnt}
\usepackage{amsmath,amssymb,amsfonts,amsthm,mathtools,bm}
\usepackage{graphicx,xcolor}
\usepackage{booktabs,tabularx,array,multirow,makecell}
\usepackage{enumitem,placeins}
\usepackage{algorithm}
\usepackage[noend]{algpseudocode}
\usepackage[hidelinks]{hyperref}
\usepackage[capitalize,nameinlink]{cleveref}
\allowdisplaybreaks
\newcolumntype{Y}{>{\centering\arraybackslash}X}
\newcolumntype{L}{>{\raggedright\arraybackslash}X}
\numberwithin{equation}{section}
\numberwithin{algorithm}{section}

\theoremstyle{plain}
\newtheorem{theorem}{Theorem}[section]
\newaliascnt{lemma}{theorem}
\newtheorem{lemma}[lemma]{Lemma}
\aliascntresetthe{lemma}
\newaliascnt{corollary}{theorem}
\newtheorem{corollary}[corollary]{Corollary}
\aliascntresetthe{corollary}
\newaliascnt{proposition}{theorem}

\aliascntresetthe{proposition}
\theoremstyle{definition}
\newaliascnt{definition}{theorem}

\aliascntresetthe{definition}
\newaliascnt{assumption}{theorem}
\newtheorem{assumption}[assumption]{Assumption}
\aliascntresetthe{assumption}
\theoremstyle{remark}
\newaliascnt{remark}{theorem}
\newtheorem{remark}[remark]{Remark}
\aliascntresetthe{remark}
\crefname{assumption}{Assumption}{Assumptions}
\Crefname{assumption}{Assumption}{Assumptions}
\crefname{remark}{Remark}{Remarks}
\Crefname{remark}{Remark}{Remarks}
\crefname{appendix}{Appendix}{Appendices}
\Crefname{appendix}{Appendix}{Appendices}

\newcommand{\R}{\mathbb{R}}

\newcommand{\E}{\mathbb{E}}
\newcommand{\Pp}{\mathbb{P}}
\newcommand{\dd}{\,\mathrm{d}}
\newcommand{\eps}{\varepsilon}
\newcommand{\wt}{\widetilde}

\newcommand{\cA}{\mathcal{A}}
\newcommand{\cB}{\mathcal{B}}
\newcommand{\cD}{\mathcal{D}}
\newcommand{\cE}{\mathcal{E}}

\newcommand{\cG}{\mathcal{G}}

\newcommand{\cR}{\mathcal{R}}
\newcommand{\cU}{\mathcal{U}}

\newcommand{\norm}[1]{\left\lVert #1\right\rVert}
\newcommand{\abs}[1]{\left|#1\right|}

\newcommand{\dist}{\operatorname{dist}}
\newcommand{\Lip}{\operatorname{Lip}}
\newcommand{\argmin}{\operatorname*{arg\,min}}
\newcommand{\esssup}{\operatorname*{ess\,sup}}
\newcommand{\ptesssup}[2]{\esssup_{\tau\in(0,T)}
  \norm{#1(\tau,\cdot)}_{\infty,#2}}

\newcommand{\one}{\mathbf{1}}
\newcommand{\comp}{\mathrm{comp}}
\newcommand{\ext}{\mathrm{ext}}

\newcommand{\cons}{\mathrm{cons}}
\newcommand{\PI}{\mathrm{PI}}

\DeclareMathOperator{\Pois}{Pois}

\title{Monotone Neural Policy Iteration for High-Dimensional First-Order
Hamilton--Jacobi--Bellman Equations\thanks{This work was supported
by the National Research Foundation of Korea (NRF) grant funded by the
Korea government (MSIT) (\mbox{RS-2023-00219980}, \mbox{RS-2026-25490291},
\mbox{RS-2026-25499800}).}}

\author{Minseok Kim\thanks{Department of Applied Artificial Intelligence,
Seoul National University of Science and Technology, Seoul, Republic of
Korea.  Equal contribution.}
\and Yeongjong Kim\thanks{Department of Mathematics, Pohang University of
Science and Technology, Pohang, Republic of Korea.  Equal contribution.}
\and Namkyeong Cho\thanks{Department of Mathematical Finance and Big Data,
Gachon University, Seongnam, Republic of Korea
(\href{mailto:namkyeong.cho@gmail.com}{\nolinkurl{namkyeong.cho@gmail.com}}).  Co-corresponding author.}
\and Yeoneung Kim\thanks{Department of Industrial Engineering, Yonsei
University, Seoul, Republic of Korea (\href{mailto:yeoneung@yonsei.ac.kr}{\nolinkurl{yeoneung@yonsei.ac.kr}}).
Co-corresponding author.}}

\date{}
\hypersetup{
  pdftitle={Monotone Neural Policy Iteration for High-Dimensional First-Order Hamilton--Jacobi--Bellman Equations},
  pdfauthor={Minseok Kim, Yeongjong Kim, Namkyeong Cho, and Yeoneung Kim}
}

\begin{document}
\maketitle

\begin{abstract}
We analyze a neural semi-discrete method for high-dimensional first-order
Hamilton--Jacobi--Bellman (HJB) equations with known or learned dynamics.
Centered differences and an artificial viscosity $Nh=O(h)$ define a monotone
operator evaluated through $2d+1$ shifted network queries; policy iteration
solves the resulting Bellman equation without a tensor grid.
At fixed $h$, the sharp componentwise condition $\max_i|f_i|\le2N$
turns every frozen-policy operator into a nearest-neighbor Markov-chain
generator with a policy-independent total jump rate.
Uniformization gives whole-space well-posedness for measurable feedbacks,
an explicit Poisson-tail bound on the numerical domain of dependence,
and boundary-free localization.
The representation also yields a posteriori policy-evaluation bounds
that account for residual and learned-model errors.
A greedy-gap analysis controls inexact policy iteration at fixed $h$;
a separate consistency estimate connects the semi-discrete equation
to the continuous HJB equation.
Experiments reproduce the extremal tail, show rates consistent with
$O(\sqrt h)$ and nearly $h$-independent exact-policy-iteration decay,
and assess empirical estimator effectivity.
A nonsmooth example shows that the continuous residual can miss a
non-viscosity solution, whereas the shifted residual detects the defect.
Further tests provide a structured interval-verified certificate
calibration, an early-budget benefit of policy freezing for bang--bang
control, and learned-dynamics diagnostics.
A structured nonlinear problem with active compact-control constraints
is tested against a manufactured semi-discrete reference through $d=1024$.
\end{abstract}

\medskip
\noindent\textbf{Keywords:}
high-dimensional Hamilton--Jacobi--Bellman equations, deterministic optimal
control, monotone finite differences, neural policy iteration, learned
dynamics, Markov-chain approximation, a posteriori error estimates
\par

\medskip
\noindent\textbf{Mathematics Subject Classification:}
65M06, 65M12, 49L25, 68T07, 60J27
\par

\section{Introduction}
\label{sec:introduction}

Monotone finite-difference and semi-Lagrangian methods provide the natural
route to viscosity solutions of first-order Hamilton--Jacobi--Bellman
(HJB) equations, but a tensor grid with $M$ nodes per coordinate contains
$M^d$ unknowns.  Neural representations remove this array, yet a small
continuous HJB residual does not enforce the monotonicity needed to select the
viscosity solution.  We retain a neural value function while keeping the
HJB operator semi-discrete: centered differences and the vanishing
artificial viscosity $\nu_h=Nh=O(h)$ are evaluated at
$V_\theta(\tau,x)$ and $V_\theta(\tau,x\pm he_i)$, $i=1,\ldots,d$.
For a frozen feedback, the monotone stencil is a nearest-neighbor jump
generator; policy evaluation is followed by minimization of the discrete
Hamiltonian.  Evaluating the residual at one collocation point requires  $2d+1$ value queries and coordinate
streaming uses $O(Bd)$ shifted-input storage for a batch of size $B$.
We refer to this procedure as semi-discrete policy
iteration (SDPI).
The paper also treats learned dynamics: componentwise model error
controls monotonicity, whereas Euclidean model error enters the residual
and greedy-gap estimates.
\subsection{Contributions}

The shifted residual was introduced in \cite{esteve2025finite}.  The
present work develops a fixed-$h$ pointwise error theory for neural policy
iteration based on this residual.  The sharp monotonicity condition
$\max_i|f_i|\le2N$ yields
uniformization by a policy-independent rate-$\Lambda_h$ Poisson process,
whole-space Bellman verification, an
explicit numerical domain-of-dependence bound, attained by a constant
extremal drift, and coupling-based localization without
boundary traces or exterior gradients.  
A stopped-chain argument bounds the neural policy-evaluation error by
the residual, the initial-condition error, the learned-model error, and
a term that accounts for the jump chain leaving the evaluation box.
Greedy gaps accommodate nonunique minimizers and measurable policy improvements and lead to
a recursive policy-error bound with factorially weighted history for inexact policy iteration.

Numerical experiments evaluate rates consistent with $O(\sqrt h)$,
Poisson-tail sharpness and localization, estimator effectivity, nonsmooth
viscosity-solution selection, learned-model error, and structured neural
approximation through $d=1024$ with active compact-control constraints.  A
sparse-actuator Allen--Cahn problem provides a separate nonmanufactured
feasibility test.  Throughout, we distinguish analytic bounds from quantities
computed on finite validation sets.

\subsection{Scope of the Analysis}

The analysis is at fixed $h$; convergence to the continuous HJB solution
additionally requires consistency and control of error amplification as
$h\downarrow0$. Learned-model estimates use uniform error envelopes, and
localization requires compatible boxes and horizons. Approximation and
uniform validation remain dimension dependent; the high-dimensional tests
use architectures adapted to local interactions.

\section{Relation to Existing Numerical and Learning-Based HJB Methods}
\label{sec:related}

Viscosity solutions frame HJB well-posedness
\cite{crandall1983viscosity,bardi1997optimal}. For schemes, monotonicity,
stability, and consistency imply convergence under comparison
\cite{crandall1984two,barles1991convergence,oberman2006convergent};
semi-Lagrangian and Markov-chain methods use nonnegative interpolation weights
or transition probabilities
\cite{falcone2014semilagrangian,kushner2001numerical}. Sparse-grid
semi-Lagrangian schemes were demonstrated through $d=8$, although their
interpolation need not be monotone \cite{bokanowski2013adaptive}. At fixed
$h$, our reference uses the artificial-viscosity operator analyzed by Tang
et al.\ \cite{tang2025policy}. Under bounded Lipschitz data and a unique
uniformly Lipschitz minimizing selector, they prove $O(\sqrt h)$ consistency
and geometric fixed-$h$ convergence of exact policy iteration in
$L^2_{\rm loc}$. We complement this with bounded-box neural evaluation and
whole-space localization, measurable updates with nonunique minimizers,
learned dynamics, and inexact sup-norm propagation.

High-dimensional solvers exploit representation or max-plus structure
\cite{darbon2016algorithms,chow2019algorithm,darbon2020overcoming,
mceneaney2007curse}, low-rank tensors \cite{dolgov2021tensor}, or
characteristic data \cite{nakamurazimmerer2021adaptive}. Physics-informed
neural networks (PINNs) and neural
policy iteration often use continuous PDE residuals
\cite{raissi2019physics,meng2024pinnpi,kim2025stochastic,yang2025hji},
whereas deep-BSDE and operator-learning methods use stochastic trajectories
or learned maps \cite{han2018solving,lee2024deeponet}. Such residual training
does not itself define a monotone numerical scheme. The analyses in
\cite{kim2025stochastic,yang2025hji} instead concern continuous second-order
equations.

Finite-difference neural residual minimization was introduced in
\cite{esteve2025finite}. The finite-node analysis of
\cite{bokanowski2026residual} gives a posteriori bounds and conditional
optimization guarantees for residual losses over grid values. Our reference
is continuous-time and whole-space. Uniformization gives a Poisson-tail bound
on exterior influence; coupling and stopping localize without artificial
boundary data and bound evaluation error a posteriori. Howard-type policy
iteration, including continuous-time PDE variants, provides the classical
context
\cite{howard1960dynamic,puterman1979convergence,tran2025exploratory,
guo2026nonconvex}. Our Hamiltonian-gap estimates remain valid for measurable
selectors and nonunique minimizers.

\section{Problem Formulation and Lattice-Aware Norms}
\label{sec:problem}

\subsection{Finite-Horizon Control in Time-to-Go Coordinates}

Let $U$ be a nonempty compact metric space and $T>0$.  The physical
control problem is
\[
 \dot X_t=f_{\rm phys}(t,X_t,u_t),\quad u_t\in U,
 \qquad
 J(t_0,x;u)=\int_{t_0}^T L_{\rm phys}(t,X_t,u_t)\dd t+g(X_T),
\]
with running cost $L_{\rm phys}$, terminal cost $g$, and value function
$v(t_0,x)=\inf_u J(t_0,x;u)$ over measurable controls.  We work in the
time-to-go variable $\tau=T-t\in[0,T]$ and set
$V(\tau,x):=v(T-\tau,x)$,
$f(\tau,x,u):=f_{\rm phys}(T-\tau,x,u)$, and
$L(\tau,x,u):=L_{\rm phys}(T-\tau,x,u)$,
so that $f:[0,T]\times\R^d\times U\to\R^d$,
$L:[0,T]\times\R^d\times U\to\R$, and $g:\R^d\to\R$ are the data used
throughout.  Under standard assumptions, $V$ is the unique bounded
uniformly continuous viscosity solution of the initial-value problem
\begin{equation}
 \partial_\tau V(\tau,x)
 -\min_{u\in U}\left\{
 f(\tau,x,u)\cdot\nabla V(\tau,x)+L(\tau,x,u)
 \right\}=0,
 \qquad V(0,x)=g(x)
 \label{eq:continuous-hjb}
\end{equation}
\cite{bardi1997optimal}.  Time dependence of $f$ and $L$ is retained to
accommodate tracking costs and time-varying learned models.  The
fixed-$h$ theory of \cref{sec:chain,sec:localization,sec:pi}
uses \eqref{eq:continuous-hjb} only as the target of the consistency
estimate in \cref{sec:endtoend}; it requires neither a classical
solution nor the assumptions behind the viscosity characterization.

A feedback is a Borel map $\pi:[0,T]\times\R^d\to U$ in time-to-go
coordinates.  We abbreviate
$a_\pi(\tau,x)=f(\tau,x,\pi(\tau,x))$ and
$L_\pi(\tau,x)=L(\tau,x,\pi(\tau,x))$.

\subsection{Shift Operators}

For $h>0$ and a globally defined bounded function $v:\R^d\to\R$, let
\begin{align}
 D_i^{+,h}v(x)&=\frac{v(x+he_i)-v(x)}{h},
 &
 D_i^{-,h}v(x)&=\frac{v(x)-v(x-he_i)}{h},
 \nonumber\\
 D_i^{0,h}v(x)&=\frac{v(x+he_i)-v(x-he_i)}{2h},
 &
 \nabla_0^hv&=(D_1^{0,h}v,\ldots,D_d^{0,h}v),
 \nonumber\\
 \Delta_hv(x)&=\sum_{i=1}^d
 \frac{v(x+he_i)-2v(x)+v(x-he_i)}{h^2}.
 \nonumber
\end{align}
Here $v$ is a function on $\R^d$, not a vector of nodal values; $h$ is a
translation scale, not a mesh width.

Fix $N>0$ and set $\nu_h=Nh$.
For a policy $\pi$, the exact semi-discrete policy value solves
\begin{equation}
 \partial_\tau V^{h,\pi}
 -a_\pi\cdot\nabla_0^hV^{h,\pi}
 -\nu_h\Delta_hV^{h,\pi}
 =L_\pi,
 \qquad V^{h,\pi}(0,\cdot)=g.
 \label{eq:policy-evaluation}
\end{equation}
The semi-discrete Bellman equation is
\begin{equation}
 \partial_\tau V^{h,*}
 -\nu_h\Delta_hV^{h,*}
 -\min_{u\in U}
 \left\{f(\tau,x,u)\cdot\nabla_0^hV^{h,*}+L(\tau,x,u)\right\}=0,
 \qquad V^{h,*}(0,\cdot)=g.
 \label{eq:semidiscrete-bellman}
\end{equation}

\subsection{Neural Realization and Computational Footprint}
\label{sec:computational-footprint}

For a neural approximation $V_\theta:[0,T]\times\R^d\to\R$, evaluating
$\nabla_0^hV_\theta$ and $\Delta_hV_\theta$ on a batch of $B$ collocation
points uses $B(2d+1)$ scalar value queries.  Streaming the two shifts for
each coordinate uses $O(Bd)$ shifted-input storage.  Reverse-mode training
also stores network activations unless these are recomputed or checkpointed;
the high-dimensional tests use an exact local-stencil implementation for
finite-range architectures.  A tensor grid, by comparison, stores $M^d$
value unknowns for $M$ nodes per coordinate.  The query count excludes
minimization over $U$.

\subsection{Standing Data and Monotonicity Conditions}

\begin{assumption}[Bounded data and monotone envelope]
\label[assumption]{ass:bounded-data}
The data $f$ and $L$ are bounded and jointly Borel in $(\tau,x,u)$,
and for every $(\tau,x)$ the maps $u\mapsto f(\tau,x,u)$ and
$u\mapsto L(\tau,x,u)$ are continuous.  The terminal cost $g$ is bounded
and Borel.  Set
\begin{align*}
 M_{\comp}&:=\sup_{\tau,x,u}\max_{1\le i\le d}|f_i(\tau,x,u)|,
 &M_2&:=\sup_{\tau,x,u}|f(\tau,x,u)|,\\
 M_L&:=\sup_{\tau,x,u}|L(\tau,x,u)|,
 &M_g&:=\sup_x|g(x)|,
\end{align*}
and, with $h>0$ and $N>0$ fixed throughout the fixed-$h$ analysis,
\begin{equation}
 M_V:=M_g+TM_L,
 \qquad
 M_{\comp}\le 2N.
 \label{eq:monotonicity-envelope}
\end{equation}
\end{assumption}

The componentwise bound $M_{\comp}$, not the Euclidean bound $M_2$, is the sharp quantity for nonnegative neighbor rates.  The Euclidean bound is used when a Hamiltonian is perturbed in its gradient argument.

For $\phi:E\to\R^k$, write
\[
 \norm{\phi}_{\infty,E}:=\sup_{z\in E}|\phi(z)|.
\]
For a Borel map $r:(0,T)\times D\to\R^k$, where $D\subset\R^d$, the mixed
time--space bounds below use the iterated quantity
\[
 \ptesssup{r}{D}.
\]
Thus the essential supremum is taken only in time; the spatial norm is
pointwise. The spatial supremum cannot be replaced by
an essential supremum. Indeed, in $d=1$, take $U=\{0\}$, $f=g=0$, and
$L(x)=\one_{h\mathbb Z}(x)$. Although $\esssup_x|L(x)|=0$, one has
$\Delta_h\one_{h\mathbb Z}=0$ and hence
$V^h(\tau,x)=\tau\one_{h\mathbb Z}(x)$. Thus changes on spatial null sets
can change the pointwise semi-discrete value. For Polish $D$, the map
$\tau\mapsto\sup_{x\in D}|r(\tau,x)|$ is universally measurable.

\subsection{Boxes and Strips}

For $L>0$, write $Q_L=(-L,L)^d$.  For any open box $Q$, define
\begin{equation*}
 Q^h=\{x\in Q:x\pm he_i\in Q\text{ for all }i\},
 \qquad
 \partial_hQ=Q\setminus Q^h.
\end{equation*}
Fix boxes and a target set such that
\[
 K\Subset Q_L^h,
 \qquad Q_L\Subset Q_{L'}^h,
\]
and define the positive coordinate margins
\begin{equation*}
\begin{aligned}
 D_L&:=\inf_{x\in K}\min_{1\le i\le d}
 \dist\bigl(x_i,\R\setminus \operatorname{proj}_i(Q_L^h)\bigr),\\
 D_{L,L'}&:=\inf_{x\in Q_L}\min_{1\le i\le d}
 \dist\bigl(x_i,\R\setminus \operatorname{proj}_i(Q_{L'}^h)\bigr).
\end{aligned}
\end{equation*}
For axis-aligned boxes these are the infimal coordinatewise distances to the
indicated complements.  This is the natural geometry for the coordinatewise
numerical domain-of-dependence bound.

\section{Uniformized Jump-Chain and Bellman Theory}
\label{sec:chain}

\subsection{Generator Decomposition}

For a bounded Borel drift $a:[0,T]\times\R^d\to\R^d$ satisfying
$|a_i|\le2N$ and a bounded Borel function $v:\R^d\to\R$, set
\[
 \cB_{a,\tau}^hv(x)=a(\tau,x)\cdot\nabla_0^hv(x)+Nh\Delta_hv(x).
\]
A direct expansion yields
\begin{equation}
 \cB_{a,\tau}^hv(x)
 =\sum_{i=1}^d\lambda_i^+(\tau,x)[v(x+he_i)-v(x)]
 +\sum_{i=1}^d\lambda_i^-(\tau,x)[v(x-he_i)-v(x)],
 \label{eq:generator-rates}
\end{equation}
where
\begin{equation*}
 \lambda_i^\pm(\tau,x)=\frac{N\pm a_i(\tau,x)/2}{h}\ge0,
 \qquad
 \lambda_i^+(\tau,x)+\lambda_i^-(\tau,x)=\frac{2N}{h}.
\end{equation*}
The total rate is the constant
\begin{equation}
 \Lambda_h=\sum_{i=1}^d(\lambda_i^++\lambda_i^-)=\frac{2dN}{h}.
 \label{eq:total-rate}
\end{equation}
Define
\begin{equation*}
 P_{a,\tau}(x,\dd y)
 :=\sum_{i=1}^d\frac{\lambda_i^+(\tau,x)}{\Lambda_h}
 \delta_{x+he_i}(\dd y)
 +\sum_{i=1}^d\frac{\lambda_i^-(\tau,x)}{\Lambda_h}
 \delta_{x-he_i}(\dd y).
\end{equation*}
Then $P_{a,\tau}$ is a Borel Markov kernel and
\begin{equation*}
 \cB_{a,\tau}^hv=\Lambda_h(P_{a,\tau}v-v),
 \qquad
 (P_{a,\tau}v)(x):=\int_{\R^d}v(y)P_{a,\tau}(x,\dd y).
\end{equation*}

The two neighbor coefficients in direction $i$ are
$(N+a_i/2)/h$ and $(N-a_i/2)/h$. Hence they are nonnegative for every
$(\tau,x)$ if and only if $|a_i(\tau,x)|\le2N$, showing that the
componentwise envelope is sharp for this stencil.

\subsection{Uniformized Chain and Bellman Representation}

Fix $(\tau,x)\in[0,T]\times\R^d$ and a Borel policy $\pi$. Set $Y_0=x$
and construct the chain in elapsed time $s\in[0,\tau]$, evaluating its
transition probabilities at the remaining time $\tau-s$. Let
$\{S_k\}_{k\ge1}$ be the jump times of a Poisson process of rate
$\Lambda_h$. At $S_k=s$, choose a coordinate $I_k$ uniformly from
$\{1,\ldots,d\}$ and, conditionally on the current state $y$, choose a sign
$\Sigma_k\in\{-1,+1\}$ with
\begin{equation*}
 \Pp(\Sigma_k=+1\mid I_k=i,Y_{s-}=y)
 =\frac{N+a_{\pi,i}(\tau-s,y)/2}{2N}.
\end{equation*}
Set $Y_s=Y_{s-}+h\Sigma_ke_{I_k}$ at the jump and keep $Y$ constant between
jumps. Evaluating the coefficients at $\tau-s$ matches the initial-value
formulation in the time-to-go variable.

The construction is pathwise well defined because the number of jumps is
almost surely finite and the mark probabilities are Borel. We denote its law
and expectation by $\Pp_x^{\tau,\pi}$ and $\E_x^{\tau,\pi}$.

Let $B_b(E)$ denote the bounded Borel functions on $E$. We consider operator
families $A_\tau:B_b(\R^d)\to B_b(\R^d)$ that preserve bounded joint
measurability: $(\tau,x)\mapsto A_\tau z(\tau,\cdot)(x)$ has this property
whenever $z$ does. Given $z_0\in B_b(\R^d)$ and bounded jointly measurable
$\varphi,z$ on $[0,T]\times\R^d$, we call $z$ a pointwise integral solution of
$\partial_\tau z-A_\tau z=\varphi$, $z(0)=z_0$, if
\begin{equation*}
 z(\tau,x)=z_0(x)+\int_0^\tau
 [A_s z(s,\cdot)(x)+\varphi(s,x)]\dd s
\end{equation*}
at every $(\tau,x)$. For $A_s=\cB_{a,s}^h$, let
$\cU^a_{\tau,r}:B_b(\R^d)\to B_b(\R^d)$ be given by
$(\cU^a_{\tau,r}\psi)(x):=\E_x^{\tau,a}[\psi(Y_{\tau-r})]$,
$0\le r\le\tau$, using the same chain construction with drift $a$.
A bounded integral subsolution satisfies the pointwise inequality
\[
 z(\tau,x)\le (\cU^a_{\tau,0}z_0)(x)
 +\int_0^\tau(\cU^a_{\tau,s}\varphi(s,\cdot))(x)\dd s;
\]
the reverse inequality defines a supersolution.

We collect the representation, comparison, and Bellman verification results
used below.
\begin{theorem}
\label{thm:policy-bellman}
Under \cref{ass:bounded-data} and \eqref{eq:monotonicity-envelope}, the following hold.
\begin{enumerate}[label=(\roman*),leftmargin=2.2em]
\item Every Borel feedback $\pi$ has a unique bounded pointwise integral
solution $V^{h,\pi}$ of \eqref{eq:policy-evaluation}. For each $(\tau,x)$,
let $Y$ be the chain constructed above. Then
\begin{equation}
 V^{h,\pi}(\tau,x)
 =\E_x^{\tau,\pi}\left[
 g(Y_\tau)+\int_0^\tau
 L_\pi(\tau-s,Y_s)\dd s
 \right].
 \label{eq:policy-representation}
\end{equation}
Moreover $\sup_{(\tau,x)\in[0,T]\times\R^d}|V^{h,\pi}(\tau,x)|\le M_V$.

\item For any bounded Borel drift $a$ satisfying $|a_i|\le2N$ and bounded
Borel functions $z_0$ and $\varphi$, if $z$ is a bounded pointwise integral
subsolution of
\[
 \partial_\tau z-\cB_{a,\tau}^hz\le \varphi,
 \qquad z(0,\cdot)\le z_0,
\]
then
\begin{equation*}
 \sup_x z(\tau,x)
 \le \sup_x z_0(x)+\int_0^\tau\sup_x\varphi(s,x)\dd s.
\end{equation*}
The analogous lower comparison holds for supersolutions.

\item The semi-discrete Bellman equation \eqref{eq:semidiscrete-bellman} has
a unique bounded pointwise integral solution, and
\begin{equation}
 V^{h,*}(\tau,x)=\inf_{\pi}V^{h,\pi}(\tau,x),
 \label{eq:bellman-verification}
\end{equation}
where the infimum is over Borel feedbacks. There exists a Borel selector
\[
 \pi^*(\tau,x)\in\argmin_{u\in U}
 \{f(\tau,x,u)\cdot\nabla_0^hV^{h,*}(\tau,x)+L(\tau,x,u)\}
\]
such that $V^{h,\pi^*}=V^{h,*}$.
\end{enumerate}
\end{theorem}

\begin{proof}
Fix a drift $a$ with $|a_i|\le2N$, and write $P_s=P_{a,s}$ and
$\Lambda=\Lambda_h$. Conditioning on the number and ordered times of
Poisson jumps gives, for $\psi\in B_b(\R^d)$,
\[
 \cU^a_{\tau,r}\psi
 =e^{-\Lambda(\tau-r)}\sum_{k=0}^{\infty}\Lambda^k
 \int_{r<s_1<\cdots<s_k<\tau}
 P_{s_k}\cdots P_{s_1}\psi\dd s_1\cdots\dd s_k,
\]
The $k=0$ integral is $\psi$. The $k$th summand has norm at most
$(\Lambda T)^k\norm{\psi}_{\infty,\R^d}/k!$.
Thus the series is uniformly convergent
and jointly Borel in $(\tau,r,x)$; its operators are positive and preserve
constants. Termwise integration gives
\[
 \cU^a_{\tau,r}\psi=\psi+\int_r^\tau
 \cB_{a,s}^h\cU^a_{s,r}\psi\dd s.
\]
Consequently, Fubini's theorem shows that
$z(\tau)=\cU^a_{\tau,0}z_0+
\int_0^\tau\cU^a_{\tau,s}\varphi(s,\cdot)\dd s$
solves the linear integral equation. Since
$\norm{\cB_{a,s}^hv}_{\infty,\R^d}
\le2\Lambda\norm{v}_{\infty,\R^d}$, Gronwall's inequality gives
uniqueness. Taking $a=a_\pi$, $z_0=g$, and $\varphi=L_\pi$ yields
\eqref{eq:policy-representation} and the bound $M_V$.
Positivity and preservation of constants give (ii).

For (iii), set
$F_s[v](x)=Nh\Delta_hv(x)+
\min_{u\in U}\{f(s,x,u)\cdot\nabla_0^hv(x)+L(s,x,u)\}$.
Compactness of $U$ and continuity in $u$ make the minimum Borel, and
\[
 \norm{F_s[v]-F_s[w]}_{\infty,\R^d}
 \le2\Lambda\norm{v-w}_{\infty,\R^d}.
\]
Starting with $v_0(\tau,x)=g(x)$, form the pointwise Picard iterates
\[
 v_{m+1}(\tau,x)=g(x)+\int_0^\tau F_s[v_m(s,\cdot)](x)\dd s.
\]
They are jointly Borel, and their successive differences are bounded by
$C(2\Lambda\tau)^m/m!$, where
$C=\sup_{\tau\le T}\norm{v_1(\tau)-v_0(\tau)}_{\infty,\R^d}<\infty$.
Uniform convergence gives a bounded solution $V^{h,*}$; the same Lipschitz
bound gives uniqueness. The objective in the Bellman minimum is jointly
Borel and continuous in $u$, so the measurable minimum theorem provides
$\pi^*$. Linear uniqueness then gives $V^{h,\pi^*}=V^{h,*}$.
For any $\pi$, subtracting the Bellman equation from the policy equation
and freezing at $\pi$ gives zero initial value and a nonnegative source for
$V^{h,\pi}-V^{h,*}$, since the minimum is at most the Hamiltonian at $\pi$.
Linear comparison proves $V^{h,\pi}\ge V^{h,*}$ and hence
\eqref{eq:bellman-verification}.
\end{proof}

\Cref{app:chain-proof} expands the ordered-series construction and its evolution identities.

\subsection{Explicit Poisson-Tail Bound on the Numerical Domain of Dependence}

Let $N_T^{(i)}$ be the number of jumps in coordinate $i$ before elapsed time $T$.  By independent thinning of the rate-$\Lambda_h$ Poisson process,
\begin{equation}
 N_T^{(i)}\sim\Pois\left(\frac{2NT}{h}\right),
 \label{eq:coordinate-poisson}
\end{equation}
independently across $i$, although the signs are state and policy dependent.
Since each coordinate jump has magnitude $h$,
$\sup_{0\le s\le T}|Y_s^{(i)}-x_i|\le hN_T^{(i)}$.
Set
\begin{equation*}
 p_h(R,T):=\Pp\left(
 \Pois\left(\frac{2NT}{h}\right)
 \ge \left\lceil\frac{R}{h}\right\rceil
 \right),
 \qquad
 \Psi_h(R,T):=1-[1-p_h(R,T)]^d.
\end{equation*}

The following estimate controls the displacement of the representation chain
uniformly over policies.
\begin{lemma}
\label{lem:numerical-dependence}
For every Borel feedback, every $x\in\R^d$, and every $R>0$,
\begin{equation*}
 \Pp_x^{\tau,\pi}\left(
 \sup_{0\le s\le\tau}\max_{1\le i\le d}|Y_s^{(i)}-x_i|\ge R
 \right)
 \le \Psi_h(R,T),
 \qquad 0\le\tau\le T.
\end{equation*}
If $R\ge4eNT$, then
\begin{equation*}
 \Psi_h(R,T)\le d\,2^{-R/h}.
\end{equation*}
\end{lemma}

\begin{proof}
The first claim follows from \eqref{eq:coordinate-poisson} and independence
of the coordinate counts.  At $\tau=T$, equality holds for the constant
extremal drift $a_i\equiv2N$, for which every jump is positive.  Let
$m=\lceil R/h\rceil$ and $\theta=2NT/h$.  Since
$\Psi_h\le dp_h$, the Poisson Chernoff inequality gives, when
$R\ge4eNT$,
\[
 \Pp(\Pois(\theta)\ge m)
 \le\left(\frac{e\theta}{m}\right)^m
 \le2^{-m}\le2^{-R/h}.
\]
\end{proof}

The chain can make arbitrarily many jumps with positive probability, so its
strict numerical domain of dependence is unbounded.  The preceding estimate
instead gives an explicit exponentially small influence tail beyond an
$O(NT)$ coordinate margin.

\section{Boundary-Free Localization and Bounded Extensions}
\label{sec:localization}

\subsection{Sensitivity to Bounded Coefficient Extensions}

Here, the whole-space chain requires globally bounded coefficients, so unbounded
physical data must be replaced by bounded extensions.  We first quantify how
the value on the inner box depends on this choice.  If two extensions agree on
the outer region specified below, chains driven by a common policy can be
coupled to coincide until they leave that region; only the exit event
contributes to the difference.  For the present boxes,
$Q_L\Subset Q_{L'}^h$ is equivalent to $L+h<L'$, and
$D_{L,L'}=L'-L-h>0$.  Compatibility with the physical problem and the horizon
is addressed in \cref{rem:extension-compatibility}.

Let $(f^{(r)},L^{(r)},g^{(r)})$, $r\in\{1,2\}$, satisfy \cref{ass:bounded-data} with common bounds $M_2,M_L,M_g$ and the same $N$ satisfying \eqref{eq:monotonicity-envelope}.  For a common feedback $\pi$, let $V^{h,\pi,(r)}$ denote the corresponding value.

\begin{theorem}
\label{thm:extension-invariance}
Assume
\begin{equation}
 f^{(1)}=f^{(2)},\quad L^{(1)}=L^{(2)}
 \quad\text{on }[0,T]\times Q_{L'}^h\times U,
 \qquad
 g^{(1)}=g^{(2)}\quad\text{on }Q_{L'}.
 \label{eq:extension-agreement}
\end{equation}
Then every common Borel feedback satisfies
\begin{equation}
 \norm{V^{h,\pi,(1)}-V^{h,\pi,(2)}}_{\infty,[0,T]\times Q_L}
 \le 2M_V\,\Psi_h(D_{L,L'},T).
 \label{eq:extension-invariance}
\end{equation}
The same bound holds for the two optimal value functions.
\end{theorem}

\begin{proof}
Both chains have total rate $\Lambda_h$, and their transition kernels agree on
$Q_{L'}^h$.  Hence they can be coupled to have identical paths until their
common exit time
$\sigma=\inf\{s\ge0:Y_s\notin Q_{L'}^h\}$.
On $\{\sigma>\tau\}$, both running and terminal costs agree.  On $\{\sigma\le\tau\}$, the absolute difference between the two total remaining payoffs is at most $2M_V$.  Hence
\[
 |V^{h,\pi,(1)}(\tau,x)-V^{h,\pi,(2)}(\tau,x)|
 \le2M_V\Pp_x(\sigma\le\tau).
\]
For $x\in Q_L$, exit from $Q_{L'}^h$ requires a coordinate displacement of at least $D_{L,L'}$, so \cref{lem:numerical-dependence} gives \eqref{eq:extension-invariance}.  For the optimal value functions, apply the policy estimate to an optimal selector of each problem and use the two resulting one-sided inequalities.
\end{proof}

To apply the theorem, one needs a bounded extension that remains unchanged on
the outer box, with parameters compatible with the box and horizon.
\begin{remark}
\label[remark]{rem:extension-compatibility}
For $C>0$, let $\operatorname{sat}_C$ denote projection onto $[-C,C]$.
Choose $N$ and $C$ so that $|f_i|\le2N$, $|L|\le C$, and $|g|\le C$ on the
respective agreement sets in \eqref{eq:extension-agreement}.  Then
$f_i^{\ext}=\operatorname{sat}_{2N}(f_i)$,
$L^{\ext}=\operatorname{sat}_C(L)$, and
$g^{\ext}=\operatorname{sat}_C(g)$ are bounded and Borel, preserve the required
$u$-continuity, and satisfy \cref{ass:bounded-data} and
\eqref{eq:monotonicity-envelope}.  These are fixed-$h$ extensions; they need
not preserve the selector regularity required by a separate consistency
theorem.  In particular, saturating the entire running cost can flatten its
control dependence.  For control-affine dynamics with a quadratic control
penalty, one may instead extend only the state-dependent parts, writing
$f^{\ext}=b^{\ext}+Bu$ and
$L^{\ext}=q^{\ext}+u^\top Ru/2$ with $R\succ0$.  This retains a unique
globally Lipschitz minimizing selector on a compact convex control set.

These choices are not independent.  For the one-dimensional drift
$f_{\rm phys}(x)=ax$, $a>0$, agreement on $Q_{L'}^h$ requires
$2N\ge a(L'-h)$.  Combining this with the sufficient tail condition
$D_{L,L'}\ge4eNT$ gives $L\le(1-2eaT)(L'-h)$; hence no admissible pair with
$L>0$ exists when $2eaT\ge1$.  Drift-aware exit estimates, Lyapunov
confinement, or horizon splitting can weaken this restriction for structured
systems.
\end{remark}

\subsection{Localized Policy Class}

Fix a Borel default feedback $\bar u:[0,T]\times\R^d\to U$.  Define
\begin{equation*}
 \Pi_L=
 \left\{
 \pi:[0,T]\times\R^d\to U\text{ Borel}:
 \pi=\bar u\text{ on }[0,T]\times(\R^d\setminus Q_L^h)
 \right\}.
\end{equation*}
The localized optimal value function is
\begin{equation*}
 V_L^{h,*}(\tau,x)=\inf_{\pi\in\Pi_L}V^{h,\pi}(\tau,x).
\end{equation*}
It is the unique bounded pointwise integral solution of the Bellman equation inside $Q_L^h$ and the frozen default-policy equation outside:
\begin{equation}
\begin{aligned}
 \partial_\tau V_L^{h,*}-Nh\Delta_hV_L^{h,*}
 &-\min_{u\in U}\{f(\tau,x,u)\cdot\nabla_0^hV_L^{h,*}+L(\tau,x,u)\}=0,
 &&x\in Q_L^h,\\
 \partial_\tau V_L^{h,*}-Nh\Delta_hV_L^{h,*}
 &-f(\tau,x,\bar u)\cdot\nabla_0^hV_L^{h,*}-L(\tau,x,\bar u)=0,
 &&x\notin Q_L^h,
\end{aligned}
\label{eq:localized-bellman-equation}
\end{equation}
with initial value $g$.  Its right-hand side is jointly Borel and
$2\Lambda_h$-Lipschitz in the supremum norm. The proof of
\cref{thm:policy-bellman} applies with a minimizing selector on $Q_L^h$
and the fixed feedback $\bar u$ outside. In particular,
\begin{equation}
 V_L^{h,*}\le V^{h,\pi}
 \qquad(\pi\in\Pi_L).
 \label{eq:localized-order}
\end{equation}

Since any two policies in $\Pi_L$ coincide with $\bar u$ outside $Q_L^h$, the source of their frozen difference equation vanishes there; comparison then bounds policy differences entirely through interior Hamiltonian data.

\subsection{Coupling Localization of the Optimal Value Function}

Recall that $K\Subset Q_L^h$ is the target region and $D_L>0$ is its
coordinatewise margin to the complement of $Q_L^h$.  A direct comparison of
the two equations would require an exterior source involving the default
policy and the crude bound
$\norm{\nabla_0^hV^{h,*}}_{\infty,[0,T]\times\R^d}=O(h^{-1})$.  The coupling below avoids this
factor and depends only on the payoff range and the exit probability.
\begin{theorem}
\label{thm:policy-localization}
For every $(\tau,x)\in[0,T]\times K$,
\begin{equation}
 0\le V_L^{h,*}(\tau,x)-V^{h,*}(\tau,x)
 \le 2M_V\,\Psi_h(D_L,T).
 \label{eq:policy-localization}
\end{equation}
In particular, if $D_L\ge4eNT$, then
\begin{equation*}
 0\le V_L^{h,*}-V^{h,*}
 \le2dM_V2^{-D_L/h}
 \quad\text{on }[0,T]\times K.
\end{equation*}
\end{theorem}

\begin{proof}
The lower bound follows from $\Pi_L$ being a subclass of all Borel feedbacks.  Let $\pi^*$ be a global optimal selector and define its localized truncation by
\[
 \pi^L(\tau,x)=
 \begin{cases}
   \pi^*(\tau,x), & x\in Q_L^h,\\
   \bar u(\tau,x), & x\notin Q_L^h.
\end{cases}
\]
The chains under $\pi^*$ and $\pi^L$ have the same total rate and transition
kernels while they remain in $Q_L^h$.  Hence they can be coupled to have
identical paths until their common first exit from $Q_L^h$.  Thus
\[
 V_L^{h,*}-V^{h,*}
 \le V^{h,\pi^L}-V^{h,\pi^*}
 \le2M_V\Pp_x\{\text{exit }Q_L^h\text{ before }\tau\}.
\]
For $x\in K$, this exit requires a coordinate displacement of at least $D_L$.
Apply \cref{lem:numerical-dependence}.
\end{proof}

\section{Learned-Dynamics Semi-Discrete Policy Iteration}
\label{sec:algorithm}

\subsection{Learned Model and Robust Monotonicity}

When the dynamics are known, we set $\wt f_n=f$, omit the identification
step, and set all model-error envelopes to zero. Otherwise, at iteration
$n$, we form identification data
$\mathcal D_n^{\rm id}=\{(\tau_j,x_j,u_j,y_j)\}_{j=1}^{M_n}$ with
$y_j\approx f(\tau_j,x_j,u_j)$.  Labels may be direct vector-field queries
or difference quotients from short controlled transitions.  A parametric model
$\wt f_{\theta}$ can be fitted by supervised regression, for example by
minimizing $M_n^{-1}\sum_j|\wt f_{\theta}(\tau_j,x_j,u_j)-y_j|^2$.
Denote the resulting model by
$\wt f_n:[0,T]\times\R^d\times U\to\R^d$ and assume it is jointly Borel
and continuous in $u$ for each $(\tau,x)$.
The model-error bounds below include observation noise and
time-discretization bias.
Define the pointwise componentwise all-control model-error envelope on a
domain $D$ by
\begin{equation*}
 \eta_{n,\comp}^{\rm all}(D)
 :=\sup_{\tau\in[0,T]}\sup_{x\in D}\sup_{u\in U}
 \max_{1\le i\le d}|\wt f_{n,i}(\tau,x,u)-f_i(\tau,x,u)|.
\end{equation*}
A sufficient condition for pointwise monotonicity of both the learned and
true operators on $[0,T]\times D$ is
\begin{equation}
 \sup_{\tau\in[0,T]}\sup_{x\in D}\sup_{u\in U}
 \max_{1\le i\le d}|\wt f_{n,i}(\tau,x,u)|
 +\eta_{n,\comp}^{\rm all}(D)
 \le2N.
 \label{eq:robust-monotonicity}
\end{equation}
Indeed, $|f_i|\le|\wt f_{n,i}|+\eta_{n,\comp}^{\rm all}(D)\le2N$,
so both sets of neighbor rates are nonnegative. With known dynamics this
reduces to \eqref{eq:monotonicity-envelope}.

For Hamiltonian perturbations, define the Euclidean envelopes
\begin{align}
 \eta_n^{\pi}
 &:={\esssup}_{\tau\in(0,T)}\sup_{x\in Q_{L'}^h}
 \abs{\wt f_n(\tau,x,\pi_n(\tau,x))-f(\tau,x,\pi_n(\tau,x))},
 \nonumber\\
 \bar\eta_n
 &:={\esssup}_{\tau\in(0,T)}\sup_{x\in Q_L^h}\sup_{u\in U}
 \abs{\wt f_n(\tau,x,u)-f(\tau,x,u)}.
 \nonumber
\end{align}
The bound $\eta_n^\pi$ enters the residual estimate
\eqref{eq:residual-transfer}, while $\bar\eta_n$ enters the
greedy-gap estimate in \cref{cor:true-gap}.

\subsection{Approximate Policy Evaluation and Improvement}

For the current policy $\pi_n$ and model $\wt f_n$, let $\wt V_n$
denote the neural candidate produced by approximate policy evaluation on
$(0,T)\times Q_{L'}^h$.  It is trained with the learned-model frozen-policy
residual and is intended to approximate the true-model value $V^{h,\pi_n}$;
the quantities below measure the resulting evaluation error.

\begin{assumption}[Regularity of $\wt V_n$]
\label[assumption]{ass:candidate}
For each $n$, $\wt V_n$ is continuous on
$[0,T]\times\overline Q_{L'}$ and absolutely continuous in $\tau$ for every fixed $x$.  The quantities $\partial_\tau\wt V_n$, $\nabla_0^h\wt V_n$, and $\Delta_h\wt V_n$ are bounded and jointly Borel on $(0,T)\times Q_{L'}^h$.
\end{assumption}

The learned-model residual for policy $\pi_n\in\Pi_L$ is
\begin{equation}
 \wt\cR_{n,h}[\wt V_n]
 :=\partial_\tau\wt V_n
 -\wt f_n(\tau,x,\pi_n)\cdot\nabla_0^h\wt V_n
 -Nh\Delta_h\wt V_n
 -L(\tau,x,\pi_n).
 \label{eq:learned-residual}
\end{equation}
The residual with respect to the true frozen-policy operator is
\begin{equation*}
 q_n
 :=\partial_\tau\wt V_n
 -f(\tau,x,\pi_n)\cdot\nabla_0^h\wt V_n
 -Nh\Delta_h\wt V_n
 -L(\tau,x,\pi_n).
\end{equation*}
Also set
\begin{equation*}
 r_n(x)=\wt V_n(0,x)-g(x),
 \qquad
 \cA_n=\norm{\wt V_n}_{\infty,[0,T]\times\partial_hQ_{L'}}.
\end{equation*}
Subtracting the learned- and true-model residuals yields
\begin{align}
 \ptesssup{q_n}{Q_{L'}^h}
 \le{}&
 \ptesssup{\wt\cR_{n,h}[\wt V_n]}{Q_{L'}^h}
 \nonumber\\
 &+\eta_n^\pi\ptesssup{\nabla_0^h\wt V_n}{Q_{L'}^h}.
 \label{eq:residual-transfer}
\end{align}

Policy improvement is measured by the Hamiltonian greedy gap.  For a model
$F$, a function $v$, and a Borel policy $\pi$, define
\begin{align}
 \Gamma_F(v,\pi)(\tau,x)
 &:=[F(\tau,x,\pi(\tau,x))\cdot\nabla_0^hv(\tau,x)
      +L(\tau,x,\pi(\tau,x))]
 \nonumber\\
 &\quad-
 \min_{u\in U}\{F(\tau,x,u)\cdot\nabla_0^hv(\tau,x)+L(\tau,x,u)\}.
 \nonumber
\end{align}
It is nonnegative and remains well defined when the minimizer is nonunique.
For a scalar tolerance $\cG_n\ge0$, the implemented improvement satisfies
\begin{equation}
 0\le\Gamma_{\wt f_n}(\wt V_n,\pi_{n+1})\le\cG_n
 \quad\text{on }(0,T)\times Q_L^h,
 \qquad
 \pi_{n+1}=\bar u\quad\text{outside }Q_L^h.
 \label{eq:inexact-improvement}
\end{equation}
Since the learned objective is Borel in $(\tau,x)$ and continuous in $u$, a Borel exact selector exists; measurable $\cG_n$-inexact selectors can likewise be chosen.

\subsection{Algorithm}

Algorithm~\ref{alg:csdpi} summarizes the iteration. Evaluation uses
$Q_{L'}^h$, improvement uses $Q_L^h$, and the final error is measured on $K$.
Validation requires uniform bounds; \cref{sec:validation} gives conditions
for obtaining them from finite samples and a verified modulus of continuity.

\begin{algorithm}[!htbp]
\caption{Semi-discrete policy iteration with known or learned dynamics}
\label{alg:csdpi}
\begin{algorithmic}[1]
\Require shift size $h>0$; parameter $N>0$; target set and boxes
$K\Subset Q_L^h$ and $Q_L\Subset Q_{L'}^h$; default policy $\bar u$;
outer iterations $N_{\PI}$
\Require known dynamics $f$ or identification data
$\{\cD_n^{\rm id}\}_{n=0}^{N_{\PI}-1}$; training data and independent
validation data
\State Set $\pi_0=\bar u$.
\For{$n=0,\ldots,N_{\PI}-1$}
    \State \textbf{Model update:}
    set $\wt f_n=f$ if the dynamics are known; otherwise fit
    $\wt f_n$ from $\cD_n^{\rm id}$ and verify
    \eqref{eq:robust-monotonicity}.
    \State \textbf{Policy evaluation:}
    train $\wt V_n$ on $(0,T)\times Q_{L'}^h$ using
    \eqref{eq:learned-residual} and the initial condition at $\tau=0$.
    \State \textbf{Evaluation validation:}
    bound the mixed norms of the residual \eqref{eq:learned-residual}
    and $\nabla_0^h\wt V_n$, the initial mismatch
    $\norm{r_n}_{\infty,Q_{L'}}$, the strip amplitude $\cA_n$,
    and the model error $\eta_n^\pi$.
    \State \textbf{Policy improvement:}
    compute $\pi_{n+1}$ satisfying
    \eqref{eq:inexact-improvement} on $Q_L^h$ and set
    $\pi_{n+1}=\bar u$ outside $Q_L^h$.
    \State \textbf{Improvement validation:}
    bound the implemented greedy gap $\cG_n$ and the all-control
    model error $\bar\eta_n$.
\EndFor
\end{algorithmic}
\end{algorithm}

\FloatBarrier

\section{A Posteriori Error Analysis of Approximate Policy Iteration}
\label{sec:pi}

The validated quantities from the preceding section control two stages of the
iteration.  We first bound the neural approximation error for the current
policy using the residual, initial mismatch, model error, and outer-strip
amplitude.  We then transfer the learned greedy gap to the true dynamics and
propagate the resulting defect through policy iteration.

\subsection{Policy-Evaluation Certificate}

For the remainder of the fixed-$h$ analysis, assume \cref{ass:bounded-data,ass:candidate}, \eqref{eq:monotonicity-envelope}, and $\pi_n\in\Pi_L$.  Define
\begin{equation*}
 V_n^h:=V^{h,\pi_n},
 \qquad
 \xi_n:=\norm{\wt V_n-V_n^h}_{\infty,[0,T]\times Q_L}.
\end{equation*}

The next estimate uses only the candidate on $Q_{L'}$ and its boundary strip.
\begin{theorem}
\label{thm:evaluation-certificate}
The neural policy-evaluation error satisfies
\begin{align}
 \xi_n
 \le{}&
 \norm{r_n}_{\infty,Q_{L'}}
 +T\ptesssup{q_n}{Q_{L'}^h}
 +(\cA_n+M_V)\Psi_h(D_{L,L'},T)
 \nonumber\\
 \le{}&
 \norm{r_n}_{\infty,Q_{L'}}
 +T\Bigl[
 \ptesssup{\wt\cR_{n,h}[\wt V_n]}{Q_{L'}^h}
 +\eta_n^\pi\ptesssup{\nabla_0^h\wt V_n}{Q_{L'}^h}
 \Bigr]
 \nonumber\\
 &+(\cA_n+M_V)\Psi_h(D_{L,L'},T).
 \label{eq:evaluation-certificate-learned}
\end{align}
\end{theorem}

\begin{proof}
Fix $(\tau,x)\in[0,T]\times Q_L$ and let $Y$ be the reverse-time
representation chain of the true frozen policy $\pi_n$.  Set
$\sigma=\inf\{s\ge0:Y_s\notin Q_{L'}^h\}$.
Between jumps, $\wt V_n(\tau-s,Y_s)$ is absolutely continuous with
derivative $-\partial_\tau\wt V_n$, while the jump increments have
compensator $\cB_{a_{\pi_n},\tau-s}^h\wt V_n\dd s$. Subtracting the
resulting finite-variation term gives a martingale up to $\sigma$.
The chain stays on the countable set
$x+h\mathbb Z^d$, so the time derivatives may be used outside a common
null set. All candidate values and drift terms up to the exit jump are
bounded. Stop first at the $m$th jump; the finite jump rate and boundedness
justify passing to the limit and taking expectations at $\sigma\wedge\tau$.
By the strong Markov property, \eqref{eq:policy-representation} gives the
corresponding exit-time identity for $V_n^h$. Subtracting it and writing
$z=\wt V_n-V_n^h$ yields
\begin{align}
 z(\tau,x)=\E_x^{\tau,\pi_n}\Bigl[
 &r_n(Y_\tau)\one_{\{\sigma>\tau\}}
 +z(\tau-\sigma,Y_\sigma)\one_{\{\sigma\le\tau\}}
 \nonumber\\
 &+\int_0^{\sigma\wedge\tau}q_n(\tau-s,Y_s)\dd s
 \Bigr].
 \label{eq:stopped-evaluation-identity}
\end{align}
On $\{\sigma>\tau\}$, $Y_\tau\in Q_{L'}$.  For $s<\sigma$,
$Y_s\in Q_{L'}^h$.  Since every neighbor of a point in $Q_{L'}^h$ belongs
to $Q_{L'}$, the first exit location satisfies
$Y_\sigma\in\partial_hQ_{L'}$, and therefore
$|z(\tau-\sigma,Y_\sigma)|\le\cA_n+M_V$.
The exit event requires a coordinate displacement of at least $D_{L,L'}$.  Apply \cref{lem:numerical-dependence}, take absolute values, and then take the supremum over $(\tau,x)$.  The second inequality follows from \eqref{eq:residual-transfer}.
\end{proof}

A final candidate can also be checked directly, without propagating the
preceding policy-iteration defects.  For a function $W$ with the regularity
specified in \cref{ass:candidate}, define its true Bellman residual
\[
 \cR_{B,h}[W]
 :=\partial_\tau W-Nh\Delta_hW
 -\min_{u\in U}\{f(\tau,x,u)\cdot\nabla_0^hW+L(\tau,x,u)\}.
\]
Set
\begin{align*}
 \mathfrak C_{L'}(W):={}&\norm{W(0,\cdot)-g}_{\infty,Q_{L'}}
 +T\ptesssup{\cR_{B,h}[W]}{Q_{L'}^h}\\
 &+\bigl(\norm{W}_{\infty,[0,T]\times\partial_hQ_{L'}}+M_V\bigr)
 \Psi_h(D_{L,L'},T).
\end{align*}
The same stopped comparison gives a direct bound for the final candidate.
\begin{corollary}
\label{cor:direct-bellman}
\[
 \norm{W-V^{h,*}}_{\infty,[0,T]\times Q_L}\le \mathfrak C_{L'}(W).
\]
\end{corollary}
\begin{proof}
Write $r=\cR_{B,h}[W]$. For any policy $\pi$, the frozen residual of $W$
is $r-\Gamma_f(W,\pi)$. Using an optimal selector $\pi^*$ from
\cref{thm:policy-bellman}, the stopped argument above and
$\Gamma_f(W,\pi^*)\ge0$ give $W-V^{h,*}\le\mathfrak C_{L'}(W)$.
Next choose a Borel minimizing selector $\pi_W$ for $W$ on $Q_{L'}^h$
and set it equal to $\bar u$ outside. Its frozen residual is $r$, so
the same argument gives $|W-V^{h,\pi_W}|\le\mathfrak C_{L'}(W)$.
Since $V^{h,\pi_W}\ge V^{h,*}$, the reverse bound follows.
\end{proof}
If the residual is formed with a learned model $\wt f$, its true-residual
term is bounded by the learned-residual term plus
$\eta_Q\ptesssup{\nabla_0^hW}{Q_{L'}^h}$, where
$\eta_Q:=\esssup_\tau\sup_{x\in Q_{L'}^h,u\in U}|\wt f-f|$,
by the Lipschitz bound for the minimum in its drift argument.
\Cref{app:stopped-proof} also bounds the cost of an approximately greedy policy extracted
from $W$.

\subsection{Transfer of the Greedy Gap}

Two estimates transfer Hamiltonian gaps between value functions and dynamics
models.
\begin{lemma}
\label{lem:gap-stability}
Let $F$ satisfy $\sup_{\tau,x,u}|F(\tau,x,u)|\le M_F$.  For bounded functions $v,w$ and a policy $\pi$,
\begin{equation}
 \Gamma_F(w,\pi)
 \le\Gamma_F(v,\pi)+2M_F|\nabla_0^h(w-v)|.
 \label{eq:gap-value-stability}
\end{equation}
For two models $F_1,F_2$ and fixed $v$,
\begin{equation}
 \Gamma_{F_1}(v,\pi)
 \le\Gamma_{F_2}(v,\pi)
 +2\sup_{u\in U}|F_1(\tau,x,u)-F_2(\tau,x,u)|\,|\nabla_0^hv|.
 \label{eq:gap-model-stability}
\end{equation}
\end{lemma}

\begin{proof}
At fixed $(\tau,x)$, write $H_F(u,p)=F(\tau,x,u)\cdot p+L(\tau,x,u)$.  Both $H_F(\pi,p)$ and $\min_uH_F(u,p)$ are $M_F$-Lipschitz in $p$, proving \eqref{eq:gap-value-stability}.  Replacing $F_2$ by $F_1$ changes each of the two terms defining the gap by at most $\sup_u|F_1-F_2||p|$, which gives \eqref{eq:gap-model-stability}.
\end{proof}

Define the true gap of the new policy relative to the exact current policy
value by
\[
 \bar g_n:=\ptesssup{\Gamma_f(V_n^h,\pi_{n+1})}{Q_L^h}.
\]
The preceding lemma transfers the validated learned-model gap to this true gap.

\begin{corollary}
\label{cor:true-gap}
Under \eqref{eq:inexact-improvement},
\begin{equation*}
 \bar g_n
 \le\cG_n
 +2\bar\eta_n\ptesssup{\nabla_0^h\wt V_n}{Q_L^h}
 +\frac{2\sqrt d M_2}{h}\xi_n.
\end{equation*}
\end{corollary}

\begin{proof}
Apply \cref{lem:gap-stability} first to transfer from $V_n^h$ to $\wt V_n$ under the true model and then to transfer from $f$ to $\wt f_n$.  For $x\in Q_L^h$ and $e=\wt V_n-V_n^h$,
\[
 |\nabla_0^he(\tau,x)|
 \le\frac{\sqrt d}{h}\norm{e(\tau,\cdot)}_{\infty,Q_L}
 \le\frac{\sqrt d}{h}\xi_n.
\]
\end{proof}

The gap estimate covers measurable, discontinuous selectors and nonunique
minimizers. Freezing $V_{n+1}^h-V_n^h$ at $\pi_{n+1}$ gives the source
$\Gamma_f(V_n^h,\pi_{n+1})-\Gamma_f(V_n^h,\pi_n)\le\bar g_n$
on $Q_L^h$ and zero outside. Comparison therefore yields
$V_{n+1}^h\le V_n^h+\tau\bar g_n$; exact greedy improvement is monotone.

\subsection{Recursive Error Bounds for Policy Iteration}

Define the exact localized policy error
\begin{equation*}
 e_n(\tau):=\sup_{x\in\R^d}\bigl(V_n^h-V_L^{h,*}\bigr)(\tau,x),
 \qquad
 \cE_n:=\sup_{0\le\tau\le T}e_n(\tau).
\end{equation*}
By \eqref{eq:localized-order}, $0\le V_n^h-V_L^{h,*}\le e_n(\tau)$.  Positivity improves the central-difference estimate: for $w(\tau,\cdot)=V_n^h(\tau,\cdot)-V_L^{h,*}(\tau,\cdot)$,
\begin{equation}
 |\nabla_0^hw(\tau,x)|
 \le\frac{\sqrt d}{2h}e_n(\tau).
 \label{eq:positive-central-estimate}
\end{equation}
Indeed, each pair $w(x+he_i),w(x-he_i)$ lies in the interval $[0,e_n(\tau)]$.

Set $\kappa_h:=\sqrt d M_2/h$ and
\begin{equation*}
 \delta_n:=
 \cG_n
 +2\bar\eta_n\ptesssup{\nabla_0^h\wt V_n}{Q_L^h}
 +\frac{2\sqrt dM_2}{h}\xi_n.
\end{equation*}
By \cref{cor:true-gap}, $\bar g_n\le\delta_n$.

The preceding estimates yield an integral recursion for the exact localized
policy error.
\begin{theorem}
\label{thm:recursive-pi-error}
For every $n\ge0$ and $\tau\in[0,T]$,
\begin{equation}
 e_{n+1}(\tau)
 \le\kappa_h\int_0^\tau e_n(s)\dd s+\tau\delta_n.
 \label{eq:recursive-pi-error}
\end{equation}
More generally, if the exact scheme satisfies the time-slice estimate
\begin{equation}
 \sup_{x\in\R^d}
 |\nabla_0^h(V_n^h-V_L^{h,*})(\tau,x)|
 \le C_{\PI}(h)e_n(\tau)
 \quad\text{for a.e. }\tau,
 \label{eq:slice-pi-gradient}
\end{equation}
then $\kappa_h$ in \eqref{eq:recursive-pi-error} can be replaced by $2M_2C_{\PI}(h)$.
\end{theorem}

\begin{proof}
Let $z=V_{n+1}^h-V_L^{h,*}\ge0$.  Freeze the difference equation at $\pi_{n+1}$.  Its source is zero outside $Q_L^h$ and equals $\Gamma_f(V_L^{h,*},\pi_{n+1})$ inside.  Comparison gives
\[
 e_{n+1}(\tau)
 \le\int_0^\tau
 \sup_{x\in Q_L^h}\Gamma_f(V_L^{h,*},\pi_{n+1})(s,x)\dd s.
\]
By \eqref{eq:gap-value-stability},
\[
 \Gamma_f(V_L^{h,*},\pi_{n+1})
 \le\Gamma_f(V_n^h,\pi_{n+1})
 +2M_2|\nabla_0^h(V_n^h-V_L^{h,*})|.
\]
Use $\bar g_n\le\delta_n$ and \eqref{eq:positive-central-estimate}.  The alternative follows directly from \eqref{eq:slice-pi-gradient}.
\end{proof}

The time-slice form in \eqref{eq:slice-pi-gradient} is essential.  A weaker estimate involving only the global space-time norm would produce a recursion in $\cE_n$, rather than the time-local integral recursion in $e_n(s)$.

Iterating the recursion gives a factorially weighted history bound.
\begin{corollary}
\label{cor:factorial-history}
For every $n\ge1$,
\begin{align}
 \cE_n
 \le{}&
 \frac{(\kappa_hT)^n}{n!}\cE_0
 +\sum_{j=0}^{n-1}
 \frac{\kappa_h^{n-1-j}T^{n-j}}{(n-j)!}\delta_j
 \nonumber\\
 \le{}&
 \frac{(\kappa_hT)^n}{n!}\cE_0
 +A_h\sup_{0\le j<n}\delta_j,
 \qquad
 A_h=
 \begin{cases}
 \dfrac{e^{\kappa_hT}-1}{\kappa_h},&\kappa_h>0,\\[6pt]
 T,&\kappa_h=0.
 \end{cases}
 \label{eq:factorial-uniform}
\end{align}
Consequently:
\begin{enumerate}[label=(\roman*),leftmargin=2em]
\item exact policy iteration satisfies $\cE_n\to0$ for every fixed $h>0$ and every $T>0$;
\item if $(\delta_n)$ is bounded and $\delta_n\to0$, then $\cE_n\to0$;
\item if $\sup_n\delta_n\le\delta$, then $\limsup_n\cE_n\le A_h\delta$.
\end{enumerate}
\end{corollary}

\begin{proof}
Iterating \eqref{eq:recursive-pi-error} and integrating the monomials gives, for $\tau\le T$,
\[
 e_n(\tau)
 \le\frac{(\kappa_h\tau)^n}{n!}\cE_0
 +\sum_{j=0}^{n-1}
 \frac{\kappa_h^{n-1-j}\tau^{n-j}}{(n-j)!}\delta_j.
\]
Take the supremum in $\tau$.  Summing the exponential series gives \eqref{eq:factorial-uniform}.  The convergence for a bounded null sequence follows from the Toeplitz convolution lemma, since the kernel $\{\kappa_h^{m-1}T^m/m!\}_{m\ge1}$ is summable.
\end{proof}

\begin{remark}[Dependence on $h$]
For $M_2>0$, the coefficient multiplying
$\sup_{j<n}\delta_j$ in \eqref{eq:factorial-uniform} is unbounded
as $h\downarrow0$. A joint limit as $n\to\infty$ and $h\downarrow0$
requires control of this amplification through the per-step errors or an
estimate such as \eqref{eq:slice-pi-gradient} with
$C_{\PI}(h)=O(1)$.
\end{remark}

\section{Total Error Decomposition and the Continuous Limit}
\label{sec:endtoend}

\subsection{Fixed-\texorpdfstring{$h$}{h} Error Decomposition}

The evaluation, policy-iteration, and localization errors combine as follows.
\begin{theorem}
\label{thm:fixed-h-total}
For every $n$,
\begin{equation}
 \norm{\wt V_n-V^{h,*}}_{\infty,[0,T]\times K}
 \le\xi_n+\cE_n+2M_V\Psi_h(D_L,T).
 \label{eq:fixed-h-total}
\end{equation}
\end{theorem}

\begin{proof}
On $K$,
\[
 \wt V_n-V^{h,*}
 =(\wt V_n-V_n^h)
 +(V_n^h-V_L^{h,*})
 +(V_L^{h,*}-V^{h,*}).
\]
Bound the three terms by $\xi_n$, $\cE_n$, and \cref{thm:policy-localization}.
\end{proof}
Combining
\cref{thm:evaluation-certificate,cor:factorial-history}
with \eqref{eq:fixed-h-total} gives the complete fixed-$h$ error estimate.

\subsection{Continuous Limit}

The following assumption relates the semi-discrete solution
$V^{h,*}$ to the viscosity solution of
\eqref{eq:continuous-hjb}.

\begin{assumption}[Consistency estimate]
\label[assumption]{ass:consistency-module}
The continuous HJB equation \eqref{eq:continuous-hjb} has a unique bounded uniformly continuous viscosity solution $V$, and there is a known function $\eps_{\cons}(h)\downarrow0$ such that
\begin{equation*}
 \norm{V^{h,*}-V}_{\infty,[0,T]\times\R^d}
 \le\eps_{\cons}(h).
\end{equation*}
\end{assumption}

Adding the consistency error gives the continuous-HJB estimate.
\begin{corollary}
\label{cor:continuous-total}
Under \cref{ass:consistency-module},
\begin{equation}
 \norm{\wt V_n-V}_{\infty,[0,T]\times K}
 \le\xi_n+\cE_n+2M_V\Psi_h(D_L,T)+\eps_{\cons}(h).
 \label{eq:continuous-total}
\end{equation}
\end{corollary}

The monotone semi-discrete theory of Tang, Tran, and
Zhang~\cite{tang2025policy} gives sufficient conditions for
\cref{ass:consistency-module}. After
the time reversal $\tau=T-t$, their centered-difference Bellman scheme has
the form \eqref{eq:semidiscrete-bellman}.  Their theorem assumes bounded
Lipschitz data, a single-valued Lipschitz minimizing selector, and
$N\ge\max\{1,M_2/2\}$, whereas $M_{\comp}\le2N$ is exactly the condition
for nonnegative neighbor rates in \eqref{eq:generator-rates}.
Under these assumptions, their consistency
theorem gives
\begin{equation}
 \eps_{\cons}(h)\le C(1+T)
 \bigl(1+\norm{\nabla V}_{L^\infty((0,T)\times\R^d)}\bigr)\sqrt h,
\label{eq:sqrt-h-consistency}
\end{equation}
where $C$ is independent of $h$.
The $O(\sqrt h)$ rate is sharp under the assumptions of
\cite{tang2025policy}.
The additional requirement $N\ge1$ and the selector assumptions enter only
their convergence and consistency proofs; the fixed-$h$ theory of
\cref{sec:chain,sec:localization,sec:pi} uses neither.  Thus
\cref{ass:consistency-module} is separate from the bounded-extension argument
of \cref{sec:localization}; an application of the cited rate must verify the
additional regularity and selector hypotheses for the chosen extensions.

For the physical interpretation of this estimate, suppose the bounded
extensions agree with the physical coefficients on
a space-time region containing every controlled trajectory starting
from $K$.  Under standard ODE well-posedness, the extended and physical
continuous values then coincide on $[0,T]\times K$, so $V$ in
\eqref{eq:continuous-total} is the physical value there.  If
$\sup_{t,u}|f_{\rm phys}(t,x,u)|\le C_0+C_1|x|$, then
$|X_s|\le (|x|+C_0s)e^{C_1s}$.
Compatibility of this trajectory bound with the outer box, the choice
of $N$, and the numerical domain-of-dependence margin is discussed in
\cref{rem:extension-compatibility}.

Finally, combining \eqref{eq:factorial-uniform} and
\eqref{eq:continuous-total} gives a sufficient condition for
convergence.  Let $h_k\downarrow0$ and $n_k\to\infty$.  Attach a superscript
$(k)$ to the candidates, boxes, and error quantities associated with $h_k$.
If
\[
\begin{gathered}
 \dfrac{(\kappa_{h_k}T)^{n_k}}{n_k!}\cE_0^{(k)}\to0,
 \qquad A_{h_k}\sup_{j<n_k}\delta_j^{(k)}\to0,
 \qquad \xi_{n_k}^{(k)}\to0,\\
 M_V^{(k)}\Psi_{h_k}(D_L^{(k)},T)\to0,
 \qquad \eps_{\cons}(h_k)\to0,
\end{gathered}
\]
then $\norm{\wt V_{n_k}^{(k)}-V}_{\infty,[0,T]\times K}\to0$.
The first two conditions control the dependence of the
policy-iteration error bound on $h$.

\section{Validation and Numerical Experiments}
\label{sec:experiments}

\subsection{From Finite Samples to Uniform Bounds}
\label{sec:validation}

The estimates in
\cref{thm:evaluation-certificate,cor:true-gap}
involve spatial suprema.  Values at finitely many validation points
give such bounds when the variation of the function between those
points is controlled.

For a compact set $D\subset\R^m$ and points
$Z_M=\{z_j\}_{j=1}^M\subset D$, define the fill distance
\begin{equation*}
 \rho(Z_M,D)
 :=
 \sup_{z\in D}\min_{1\le j\le M}|z-z_j|.
\end{equation*}

\begin{lemma}
\label{lem:deterministic-cover}
If $r:D\to\R^k$ is $\Lambda_r$-Lipschitz, then
\begin{equation*}
 \norm{r}_{\infty,D}
 \le
 \max_{1\le j\le M}|r(z_j)|
 +\Lambda_r\rho(Z_M,D).
\end{equation*}
\end{lemma}

\begin{proof}
For any $z\in D$, choose $z_j$ such that
$|z-z_j|\le\rho(Z_M,D)$ and use
\[
 |r(z)|
 \le
 |r(z_j)|+\Lambda_r|z-z_j|.
\]
\end{proof}

The same bound applies to residuals, greedy gaps, and model errors.
For model errors, take $z=(\tau,x,u)$.  \Cref{smsec:validation} gives corresponding
bounds for random validation samples
and population $L^2$ estimates, and discusses discontinuities near
policy-switching surfaces.  These bounds remain dimension dependent.

The experiments test the fixed-$h$ theory, viscosity-solution selection,
learned dynamics, and structured high-dimensional computation.  Quadratic
control penalties are retained when state-dependent costs and polynomial
drifts are bounded as in \cref{rem:extension-compatibility}; these extensions
are inactive on the stated boxes.  Unless noted otherwise, neural results use
three independent seeds.
All computations use one NVIDIA GeForce RTX~5070 GPU (12\,GB).  Reference
solves and numerical audits use double precision; neural optimization uses
single-precision tensors.  \Cref{app:reproducibility} records the complete configurations, and
\cref{smsec:additional-experiments} gives the tables and diagnostics omitted here.

\subsection{Theoretical Diagnostics}
\label{sec:exp-consistency}

The distance-function control problem $\dot x=u$, $|u|\le1$, $L=0$,
and $g(x)=|x|$, where $|\cdot|$ denotes the Euclidean norm, has the
exact nonsmooth viscosity solution
$V(\tau,x)=\max(|x|-\tau,0)$.
At the sharp monotonicity envelope
$N=M_{\comp}/2=1/2$, the final-time errors exhibit rates $0.52$--$0.55$ in
$d=1$ over the last three refinements and $0.60$--$0.62$ in $d=2$; the
plot and full mesh sequence are reported in \cref{smsec:additional-experiments}. These empirical
rates are close to $\sqrt h$. The benchmark uses $N=1/2$ and a nonunique
minimizer at $p=0$; the cited estimate \eqref{eq:sqrt-h-consistency}
assumes $N\ge1$ and a unique Lipschitz selector.

\paragraph{Exact policy iteration and the slice-gradient constant}
\label{sec:exp-pi}

The saturating problem $\dot x=u$, $|u|\le1$,
$L=\tfrac12x^2+\tfrac12u^2$, $g=\tfrac12x^2$, $T=1$ has a bang region
$|\partial_xV|>1$ and selector
$u^*=\operatorname{sat}_1(-\partial_xV)$.  From
$\pi_0\equiv0$, the error curves for
$h=0.02,0.01,0.005,0.0025$ nearly coincide and decay superlinearly to the
roundoff floor, substantially faster than the worst-case
$\kappa_h=O(h^{-1})$ bound.  The empirical unrestricted ratio on $K$
corresponding to \eqref{eq:slice-pi-gradient} scales as $h^{-1}$, whereas
the $1\%$-restricted ratio lies in $[2,15]$ with less than $6\%$ variation and the full-horizon
ratio of time integrals lies in $[1.99,9.10]$ for $n\le3$.  
These are restricted-slice and time-averaged diagnostics on $K$;
\eqref{eq:slice-pi-gradient} requires a global time-slice bound.
\Cref{smsec:additional-experiments} gives the curves and definitions.

\paragraph{Sharpness of the tail, localization, and monotonicity bounds}
\label{sec:exp-dependence}

With $2\times10^6$ paths ($d=1$, $h=0.01$, $N=0.5$, $T=1$),
\cref{fig:e3-dependence}(left) agrees with $\Psi_h$ within
$4\times10^{-4}$ for the
extremal drift $a\equiv2N$, where displacement is exactly
$h\Pois(2NT/h)$; the symmetric walk is many orders below the same bound.
The tail in \cref{lem:numerical-dependence} is therefore sharp without additional
drift information.

For a localized version of the preceding saturating problem, the gap is zero at
$N=M_{\comp}/2$ because the saturated feedback eliminates the outward
jump rate.  At $N=1$ it decays from $10^{-2}$ at margin $D_L=0.01$ to
roundoff near $D_L=0.3$, whereas the policy-uniform bound remains
$O(10)$; the bound is sharp for extremal drifts but conservative for
inward feedbacks.

\begin{figure}[tbhp]
\centering
\includegraphics[width=0.95\textwidth]{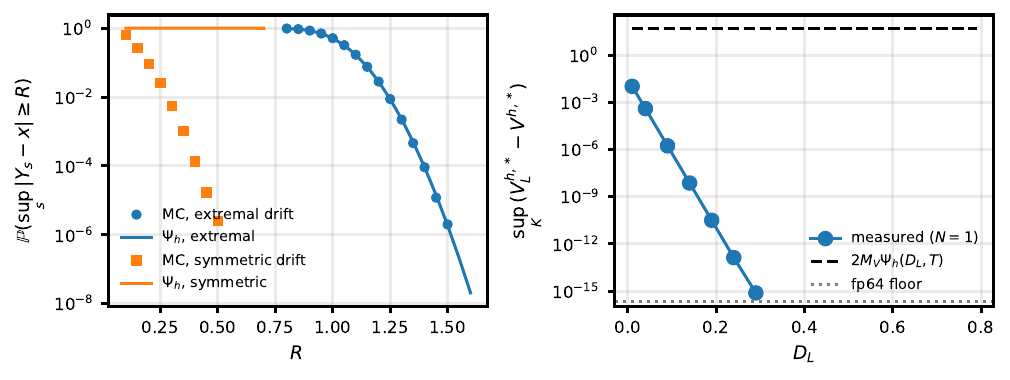}
\caption{Numerical domain of dependence.  Left: Monte Carlo displacement tails
versus $\Psi_h$; the extremal drift attains the bound, the symmetric
drift is far below it.  Right: measured localization gap
$\sup_K(V_L^{h,*}-V^{h,*})$ versus margin $D_L$ at $N=1$; at the sharp
envelope $N=M_{\comp}/2$ the measured gap is identically zero (not
plottable on the log scale).}
\label{fig:e3-dependence}
\end{figure}

For the kinked one-dimensional transport test specified in \cref{app:reproducibility},
undershoot disappears at $N=M_{\comp}/2$.
Within the monotone regime, error grows from $0.042$ at $N=0.5$ to
$0.159$ at $N=4$, so the threshold is also the most accurate tested
monotone choice for this problem.

\paragraph{Empirical estimator effectivity}
\label{sec:exp-certificate}

For the frozen feedback $\pi(x)=-\tanh(2x)$ in the same saturating problem,
we evaluate the residual-based estimator on a dense
validation grid.  Across three seeds and four training checkpoints, all
$12$ indicators exceed the corresponding errors measured on ten time
slices; at the final checkpoint the effectivity has mean $1.26$ and range
$1.10$--$1.53$. The Lipschitz correction is sample based, and the ten-slice
maximum lower-bounds the full space--time error. \Cref{smsec:additional-experiments} gives the
complete table.

We separately calibrate \cref{cor:direct-bellman} on a one-dimensional
compact-control problem with an analytic semi-discrete value.  A
$25$-parameter candidate trained only through its Bellman residual has a
verified residual bound $6.90\times10^{-5}$; Arb ball arithmetic and the
Poisson exit term give the same value for the direct certificate, compared
with a verified true-error upper bound $9.00\times10^{-6}$. This structured
calibration uses a prescribed spatial factor; details are in \cref{smsec:additional-experiments}.

\subsection{Low-Dimensional Neural Tests}
\label{sec:exp-duffing}

The Duffing problem $\dot x_1=x_2$,
$\dot x_2=-x_1-\tfrac12x_1^3+u$, $|u|\le1$,
$L=\tfrac12|x|^2+\tfrac14u^2$, $g=\tfrac12|x|^2$, $T=1$, $h=0.005$ is
solved by \cref{alg:csdpi} with the saturated bounded extension of
\cref{sec:localization} (drift core clipped to $[-(2N{-}1),2N{-}1]$,
$N=2.1=M_{\comp}/2$ on the training box, so the extension is inactive
there), and compared with a monotone grid solve of the same semi-discrete
Bellman equation on $1001^2$ nodes.  \Cref{fig:e6-duffing} shows the
result: the neural value error on $K=(-1,1)^2$ falls below
$1.5\times10^{-2}$ by the third evaluation and then plateaus near
$1.2\times10^{-2}$ (about $1\%$ of the value scale), while the greedy gap
decays through five orders of magnitude. This is consistent with persistent
evaluation error in \cref{cor:factorial-history}(iii): the final
root-mean-square residual and value error are both of order $10^{-2}$.
The comparison is qualitative, since the full defect $\delta_n$ is not
measured. Training costs
13\,s per iteration.  Across three independent seeds the final error on
$K$ is $0.89$--$1.19\times10^{-2}$ and the final greedy gap
$2\times10^{-6}$--$3\times10^{-5}$; the reported ranges include all three seeds.

To examine neural refinement, we repeated the same six-iteration SDPI run
with three seeds at $h=0.02$, $0.01$, and $0.005$, using the monotone
$h_{\mathrm{ref}}=0.0025$ solution as a fine-grid proxy.  The neural errors
against their own semi-discrete references are, respectively,
$0.0047$--$0.0071$, $0.0067$--$0.0094$, and $0.0089$--$0.0119$; the
corresponding grid-bias proxies are $0.124$, $0.0539$, and $0.0179$.
Thus the neural error remains below the grid bias, and the total error
against the fine-grid proxy decreases from $0.128$ to $0.059$--$0.061$
and then $0.0255$--$0.0266$. Thus the neural solution tracks the decreasing
grid bias over the three tested meshes; full results are in \cref{smsec:additional-experiments}.

\begin{figure}[tbhp]
\centering
\includegraphics[width=0.98\textwidth]{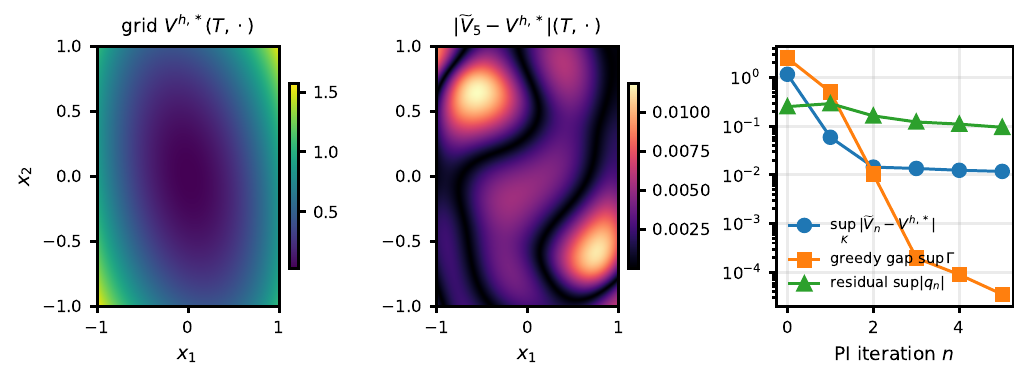}
\caption{Neural policy iteration on the Duffing problem versus the grid reference.
Left: grid-based optimal value function at $\tau=T$.  Center: absolute error of the
fifth neural iterate on $K$.  Right: value error, greedy gap, and residual at each policy iteration;
the value error plateaus while the greedy gap keeps decaying,
qualitatively consistent with the error decomposition.}
\label{fig:e6-duffing}
\end{figure}

\paragraph{Residual selection and solver comparison}
\label{sec:exp-pinn}

The eikonal problem of \cref{sec:exp-consistency} ($d=1$, $T=0.5$) admits,
besides the viscosity solution $V=\max(|x|-\tau,0)$, the spurious
almost-everywhere solution $V^{\rm sp}=|x|-\tau$, whose defect is confined
to $x=0$.  With identical networks, budgets, three sampling schemes, and
three seeds, every continuous-residual run converges to $V^{\rm sp}$
(loss $3\times10^{-8}$, value error $0.5$), while its sampled residual is
$10^{-3}$.  The shifted residual exposes an $O(1)$ defect on $|x|\le h$;
placing $30\%$ of samples in $|x|\le3h$ reduces the value error to
$0.070$--$0.096$, versus $0.41$--$0.49$ without stencil-band sampling.
Thus the shifted residual spreads the point defect at $x=0$ over the
band $|x|\le h$.  At the tested sample budget, uniform collocation
underrepresents this narrow band, so we oversample the region
$|x|\le3h$. The hard initial-condition ansatz favors the simple
$V^{\rm sp}$, so this test isolates residual detectability.
On smooth Duffing, a matched comparison uses the same $h$, $N$, network,
box, batch size, and $18000$ optimization steps.  Direct shifted-residual
training attains error $0.00688\pm0.00158$ in $70.3\pm0.6$ seconds, whereas
SDPI attains $0.01029\pm0.00150$ in $73.9\pm0.4$ seconds. Direct training
is more accurate and slightly faster here. We then remove only the quadratic
control penalty, giving the bang--bang Hamiltonian term $-|D_2^{0,h}V|$.
A design seed fixes the terminal-greedy initialization and $6000$-step
budget.  On five subsequent seeds, SDPI and direct training give errors
$0.02209\pm0.00445$ and $0.03029\pm0.00504$, respectively; SDPI is lower
in all five cases, with mean paired ratio $0.728$ (bootstrap $95\%$ interval
$[0.687,0.772]$).  Direct training with $9000$ steps still has mean error
$0.02440$, and is higher in four of five cases. Freezing the control switch
therefore improves early-budget accuracy in this bang--bang test.
Full comparisons are in \cref{smsec:additional-experiments}.
For the nonsmooth eikonal problem, residual-adaptive refinement makes the
same distinction without prescribing the kink location: the first shifted
residual refresh places $77\%$ of the new points in $|x|\le3h$, compared
with $0.4\%$ for the continuous residual.  With equal collocation budgets,
the resulting errors are $0.044\pm0.003$ and $0.500$, respectively; uniform
shifted-residual training gives $0.493$.  The full comparison and the
coarse-to-fine continuation study are in \cref{smsec:additional-experiments}.

\subsection{Learned-Dynamics Sensitivity}
\label{sec:exp-modelerr}

We next identify the Duffing drift from noisy observations using the correctly
specified polynomial library of total degree at most three, and train SDPI
with the fitted model.  Over five identification seeds, the maximum error on
a fixed state--control audit grid decays approximately as $M^{-0.605}$
with the data budget $M$.  At $M=10^2,10^3,10^4$, this finite-grid envelope
is respectively $0.0992$, $0.0241$, and $0.00538$; the corresponding
three-seed value errors are $0.0342$, $0.0144$, and $0.0131$, versus
$0.00992$ with known dynamics.  The model-error term in
\cref{thm:evaluation-certificate} exceeds the measured excess error in
all nine runs, and its robust monotonicity test passes at $M=10^4$ but not
at the two smaller budgets. All model-error terms and margin tests here use
finite-grid maxima. \Cref{smsec:additional-experiments} reports the noise model, grids, and all budgets.

\subsection{High-Dimensional Nonlinear Control with Active Constraints}
\label{sec:exp-scaling}

To obtain an exact high-dimensional test in which the compact control set is
active, consider $U=[-0.5,0.5]^d$, $T=0.75$, $h=0.01$, $N=0.75$, and
$\dot x_i=b_i(x)+u_i$, where
\[
 b_i(x)=-0.4\tanh x_i
 +0.15\bigl(\tanh x_{i-1}+\tanh x_{i+1}\bigr),
\]
with $x_0=x_1$ and $x_{d+1}=x_d$.  With $g=E$, define the state cost
$q_h$ by substituting the following value into the semi-discrete Bellman
equation, and set $L(\tau,x,u)=q_h(\tau,x)+|u|^2/2$:
\[
 V^{h,*}(\tau,x)=(1+0.3\tau)E(x)+1.5d\tau,
 \quad
 E(x)=\sum_i(1-\cos x_i)
 +\tfrac14\sum_{i=1}^{d-1}(1-\cos(x_i-x_{i+1})).
\]
Equivalently, with $p=\nabla_0^hV^{h,*}$,
$q_h=\partial_\tau V^{h,*}-Nh\Delta_hV^{h,*}-b\cdot p
-\sum_i\min_{|u_i|\le0.5}(u_ip_i+u_i^2/2)$.
The running cost is nonnegative, and the global bound
$|b_i(x)+u_i|\le1.2<2N=1.5$ leaves monotonicity margin $0.3$.  The exact
construction is used to define this cost and to evaluate the result; value
and policy targets are not used in training or checkpoint selection.

The network uses shared on-site and nearest-neighbor terms, with $3783$
parameters independent of $d$.  As shown in \cref{fig:e11-active}, the
three-seed mean relative errors of the learned correction $V-g$ at
$d=64,256,1024$ are $(6.38,4.84,6.15)\times10^{-4}$; removing the known
$1.5d\tau$ offset gives $(3.69,2.80,3.60)\times10^{-4}$.  The corresponding
policy errors are $(1.05,1.46,2.82)\times10^{-2}$.  The exact policy is saturated on
about $68.5\%$ of sampled coordinates, and the learned active-set accuracy
is $99.55\%$, $99.44\%$, and $98.89\%$.  Removing the pair term at $d=256$
raises the correction and policy errors to $0.290$ and $0.312$, respectively,
showing that the favorable scaling relies on representing the local
coupling.  At $d=1024$, a matched direct Bellman-residual run gives correction
and policy errors $5.55\times10^{-4}$ and $2.72\times10^{-2}$ in $26.8$~s,
versus $6.15\times10^{-4}$ and $2.82\times10^{-2}$ in $30.2$~s for SDPI.
The two solvers have similar accuracy, with a shorter runtime for direct
training. Standard deviations and stencil checks are in \cref{smsec:additional-experiments}.

\begin{figure}[tbhp]
\centering
\includegraphics[width=0.98\textwidth]{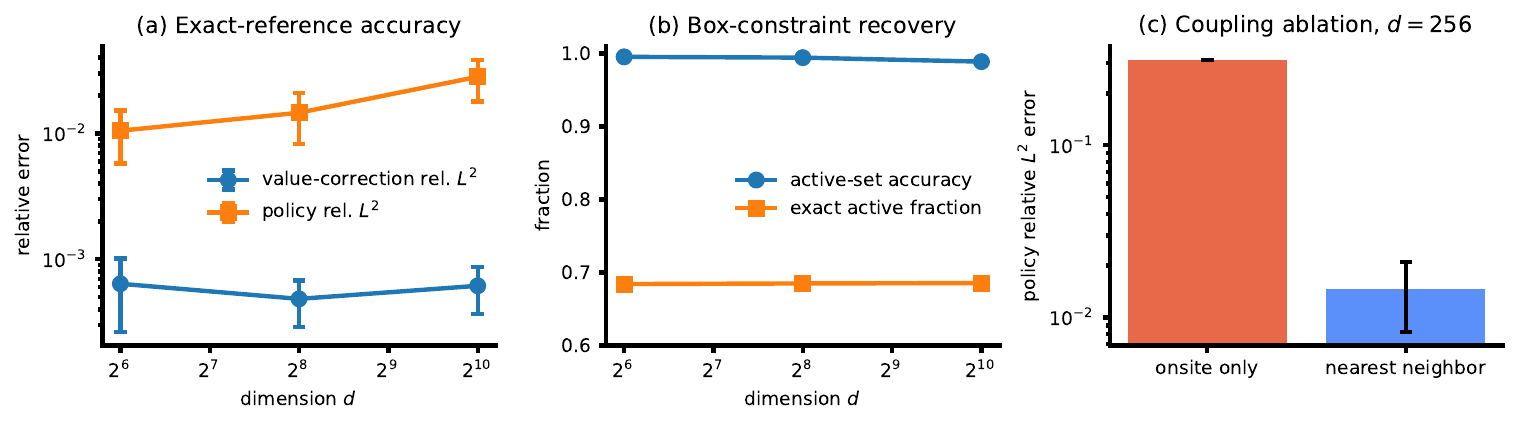}
\caption{Nonlinear active-box benchmark.  Left: correction and policy errors.
Center: exact active fraction and learned active-set accuracy.  Right:
$d=256$ coupling ablation.  The optimal semi-discrete policy is on the box
boundary for about $68.5\%$ of sampled coordinates.}
\label{fig:e11-active}
\end{figure}

\paragraph{Computational scaling and a nonmanufactured test}
\label{sec:exp-allencahn}

For a batch of $64$ and a finite-range architecture, measured peak CUDA
memory grows from $124$\,MiB at $d=64$ to $3.45$\,GiB at $d=2048$, with
log--log exponent $0.968$.  A vectorized full $(2d+1)$-query evaluation is
faster at $d=64$, essentially tied at $d=128$, and the local stencil is
$4.0$, $15.8$, and $61.7$ times faster at $d=256,512,1024$.

As a nonmanufactured feasibility test, we also consider a sparse-actuator
Allen--Cahn system with $d=1024$ and one controlled coordinate in every
eight.  Across three training seeds and ten independent test ensembles, the
SDPI cost relative to the best proportional baseline is
$0.9500\pm0.0020$.  A directly trained nonlinear local policy attains
$0.9437\pm0.0011$, about $0.65\%$ below SDPI. Both methods improve on the
proportional baseline by using local state information. The model,
baselines, radius ablation, and rollout audit are reported in \cref{smsec:additional-experiments}.

Scripts regenerate the numerical tables and figures from archived result
files at
\url{https://github.com/yeoneung/semi_discrete_pinn}.

\section{Discussion and Conclusion}
\label{sec:discussion}
The monotone shifted operator gives neural policy iteration a fixed-$h$
reference with whole-space well-posedness, Bellman verification, explicit
localization, and a posteriori error bounds. The greedy-gap recursion
controls inexact updates, and the total-error estimate states sufficient
conditions for convergence to the continuous HJB solution.
Numerically, the shifted residual detects the nonsmooth defect, the
structured direct certificate is nonvacuous, and local-interaction networks
resolve active compact controls through $d=1024$. Matched comparisons favor
direct training on smooth Duffing and policy freezing at short budgets for
bang--bang control.

The main analytical question is an $h$-uniform time-slice gradient estimate
that would sharpen simultaneous refinement and policy iteration.
Drift-aware exit estimates could improve localization for stabilizing
feedbacks. Further computational work should combine structured
architectures with uniform validation and statistical model-error bounds.

\section*{Declarations}
Generative AI tools were used for language and structural editing, to assist
with mathematical and bibliographic checks, and for code development and
preparation of computational experiments.
The authors assume responsibility for all content.

\clearpage
\appendix
\crefalias{section}{appendix}

\section{Pointwise Uniformization and Bellman Verification}
\label{app:chain-proof}

This section expands the ordered-series construction used in the proof of
\cref{thm:policy-bellman}. Integrals are evaluated pointwise in $x$;
the uniform estimates give convergence in the supremum norm
$\norm{\phi}_{\infty,E}:=\sup_{x\in E}|\phi(x)|$.

\subsection{The ordered Poisson series}

Fix a Borel feedback $\pi$, write $\cU^\pi=\cU^{a_\pi}$ in the notation
of the main text, and abbreviate
\[
 A_s:=\cB_{a_\pi,s}^h,
 \qquad P_s:=I+\Lambda_h^{-1}A_s.
\]
The same construction applies to any bounded Borel drift $a$ satisfying
$|a_i|\le2N$.
By \eqref{eq:generator-rates}--\eqref{eq:total-rate}, $P_s$ is a positive
Markov contraction on bounded Borel functions.  For $0\le r\le\tau\le T$,
define
\begin{align}
 \cU^\pi_{\tau,r}\varphi
 &:={e^{-\Lambda_h(\tau-r)}}
 \sum_{k=0}^{\infty}\Lambda_h^k
 \int_{r<s_1<\cdots<s_k<\tau}
 P_{s_k}\cdots P_{s_1}\varphi
 \,\dd s_1\cdots\dd s_k .
 \label{eq:appendix-uniformization}
\end{align}
For $k=0$, the integral is interpreted as $\varphi$.  The $k$th summand is
bounded in absolute value by
\[
 e^{-\Lambda_h(\tau-r)}
 \frac{[\Lambda_h(\tau-r)]^k}{k!}\norm{\varphi}_{\infty,\R^d}.
\]
The integrands are jointly Borel in the time variables and $x$.  Hence the
series converges absolutely and uniformly in $(\tau,r,x)$ on the closed time
simplex.  It defines a positive contraction, preserves constants, and maps
bounded Borel functions to bounded Borel functions.  Decomposition
of an ordered simplex at an intermediate time gives
\begin{equation}
 \cU^\pi_{\tau,s}\cU^\pi_{s,r}=\cU^\pi_{\tau,r},
 \qquad 0\le r\le s\le\tau\le T.
 \label{eq:evolution-property}
\end{equation}
Termwise integration, justified by the preceding Poisson majorant, yields the
pointwise integral equations
\begin{align}
 \cU^\pi_{\tau,r}\varphi
 &=\varphi+\int_r^\tau A_s\cU^\pi_{s,r}\varphi\,\dd s,
 \label{eq:left-evolution}\\
 \cU^\pi_{\tau,r}\varphi
 &=\varphi+\int_r^\tau \cU^\pi_{\tau,s}A_s\varphi\,\dd s
 \nonumber
\end{align}
first for bounded Borel $\varphi$, with each identity understood pointwise.
The operator order in \eqref{eq:appendix-uniformization} is the reason the
chain in \cref{sec:chain} uses coefficients at the remaining time.

To see the probabilistic interpretation directly, condition on $k$ Poisson
jumps in an interval of length $\tau-r$.  Their ordered elapsed times are
uniform on the corresponding simplex.  Reading the PDE times in reverse,
$s_j=\tau-q_{k+1-j}$, and applying the transition kernels in chronological
path order produces the product $P_{s_k}\cdots P_{s_1}$ in
\eqref{eq:appendix-uniformization}.  Consequently,
\begin{equation}
 (\cU^\pi_{\tau,r}\varphi)(x)
 =\E_x^{\tau,\pi}\bigl[\varphi(Y_{\tau-r})\bigr],
 \label{eq:uniformization-expectation}
\end{equation}
where the reverse-time chain uses $a_\pi(\tau-q,\cdot)$ after elapsed time
$q$.

\subsection{Frozen-policy evaluation and comparison}

Define
\begin{equation}
 V^{h,\pi}(\tau)
 :=\cU^\pi_{\tau,0}g
 +\int_0^\tau\cU^\pi_{\tau,s}L_\pi(s)\,\dd s.
 \label{eq:appendix-duhamel}
\end{equation}
The integral is pointwise well defined, the result is bounded Borel, and
\[
 \norm{V^{h,\pi}(\tau)}_{\infty,\R^d}\le M_g+\tau M_L\le M_V.
\]
Using \eqref{eq:left-evolution}, Fubini's theorem under the uniform bound, and
\eqref{eq:evolution-property}, one obtains
\begin{equation*}
 V^{h,\pi}(\tau,x)
 =g(x)+\int_0^\tau
 \bigl[A_sV^{h,\pi}(s,\cdot)(x)+L_\pi(s,x)\bigr]\dd s.
\end{equation*}
Thus $V^{h,\pi}$ is a pointwise integral solution.  If $V$ and $W$ are two such
solutions, then
\[
 \norm{V(\tau)-W(\tau)}_{\infty,\R^d}
 \le 2\Lambda_h\int_0^\tau\norm{V(s)-W(s)}_{\infty,\R^d}\dd s,
\]
because $\norm{A_s\phi}_{\infty,\R^d}
\le2\Lambda_h\norm{\phi}_{\infty,\R^d}$.
Gronwall's lemma gives uniqueness.

Equation \eqref{eq:policy-representation} follows from
\eqref{eq:uniformization-expectation} and the Markov property.  More
explicitly, the source contribution in \eqref{eq:appendix-duhamel} can be
written as
\[
 \int_0^\tau
 \E_x^{\tau,\pi}[L_\pi(\tau-q,Y_q)]\dd q
\]
after the change of variables $q=\tau-s$.

For comparison, a bounded pointwise integral subsolution with source
$\varphi$ and initial upper bound $z_0$ is defined through the
variation-of-constants inequality
\[
 z(\tau,x)
 \le (\cU^\pi_{\tau,0}z_0)(x)
 +\int_0^\tau(\cU^\pi_{\tau,s}\varphi(s))(x)\dd s
\]
at every $(\tau,x)$.  Positivity and contraction of the evolution family then
give
\[
 z(\tau,x)\le \sup_xz_0(x)+\int_0^\tau\sup_x\varphi(s,x)\dd s.
\]
Reversing the inequalities gives the supersolution comparison.

\subsection{The nonlinear Bellman equation}

For a bounded Borel function $v$, define
\begin{align}
 \mathfrak F_s[v](x)
 &:={Nh}\Delta_hv(x)
 +\min_{u\in U}
 \{f(s,x,u)\cdot\nabla_0^hv(x)+L(s,x,u)\}.
 \nonumber
\end{align}
For bounded $v,w$,
\begin{align}
 |\mathfrak F_s[v](x)-\mathfrak F_s[w](x)|
 &\le\sup_{u\in U}|\cB_{f(s,\cdot,u),s}^h(v-w)(x)|,\nonumber\\
 \norm{\mathfrak F_s[v]-\mathfrak F_s[w]}_{\infty,\R^d}
 &\le2\Lambda_h\norm{v-w}_{\infty,\R^d}.
 \label{eq:bellman-lipschitz}
\end{align}
The minimum is Borel because the objective is a Carath\'eodory integrand and
$U$ is compact.  Starting with $V^{(0)}(\tau,x)=g(x)$, define
\begin{equation*}
 V^{(m+1)}(\tau,x)
 =g(x)+\int_0^\tau\mathfrak F_s[V^{(m)}(s,\cdot)](x)\dd s.
\end{equation*}
Successive differences satisfy the standard factorial estimate
\[
 \sup_{s\le\tau}\norm{V^{(m+1)}(s)-V^{(m)}(s)}_{\infty,\R^d}
 \le \frac{(2\Lambda_h\tau)^m}{m!}C_0
\]
for a finite $C_0$ depending only on the first Picard difference.  Therefore
$V^{(m)}$ converges uniformly to a bounded Borel solution $V^{h,*}$ of the
Bellman integral equation.  Estimate \eqref{eq:bellman-lipschitz} and
Gronwall's lemma give uniqueness.

For each $(\tau,x)$, let
\begin{equation*}
 G(\tau,x,u)
 :=f(\tau,x,u)\cdot\nabla_0^hV^{h,*}(\tau,x)+L(\tau,x,u).
\end{equation*}
It is jointly Borel in $(\tau,x,u)$ and continuous in $u$ for each
$(\tau,x)$.  The measurable minimum theorem gives a Borel selector
\begin{equation*}
 \pi^*(\tau,x)\in\argmin_{u\in U}G(\tau,x,u).
\end{equation*}
No uniqueness is asserted or needed.  The Bellman integral equation is then
exactly the frozen-policy integral equation for $\pi^*$, so uniqueness yields
$V^{h,\pi^*}=V^{h,*}$.

For an arbitrary Borel policy $\pi$, set $W=V^{h,\pi}-V^{h,*}$.  Subtracting
the Bellman equation from the policy equation and freezing at $\pi$ gives
$\partial_\tau W-\cB_{a_\pi,\tau}^hW=\Gamma_f(V^{h,*},\pi)\ge0$, with
$W(0)=0$.  Positivity and comparison imply $V^{h,\pi}\ge V^{h,*}$.
Together with the equality for
$\pi^*$, this proves \eqref{eq:bellman-verification} and completes the proof
of \cref{thm:policy-bellman}.

The localized Bellman assertion in \eqref{eq:localized-bellman-equation}
follows from the same construction.  Replace $\mathfrak F_s$ by
\[
 \mathfrak F_s^L[v](x)=
 \begin{cases}
  \mathfrak F_s[v](x),&x\in Q_L^h,\\
  \cB_{a_{\bar u},s}^hv(x)+L_{\bar u}(s,x),&x\notin Q_L^h.
 \end{cases}
\]
This vector field is Borel and $2\Lambda_h$-Lipschitz in the pointwise
supremum norm.  Picard iteration gives a unique bounded pointwise integral
solution.
Choose a Borel minimizing selector on $Q_L^h$ and set it equal to $\bar u$
outside; the preceding freezing argument identifies the solution with
$\inf_{\pi\in\Pi_L}V^{h,\pi}$ and proves \eqref{eq:localized-order}.

\subsection{Two freezing consequences used in the main text}

For $\pi,\pi'\in\Pi_L$, freezing the difference equation at $\pi$ gives
the interior-source bound
\begin{align*}
 \norm{V^{h,\pi}(\tau)-V^{h,\pi'}(\tau)}_{\infty,\R^d}
 \le \int_0^\tau
 \sup_{x\in Q_L^h}
 \big|[a_\pi-a_{\pi'}]\cdot\nabla_0^hV^{h,\pi'}
       +L_\pi-L_{\pi'}\big|(s,x)\dd s,
\end{align*}
because both policies equal $\bar u$ outside $Q_L^h$ and the source of
the difference vanishes there.

Similarly, freezing $V_{n+1}^h-V_n^h$ at $\pi_{n+1}$ shows almost-monotone
improvement: the source on $Q_L^h$ is
$\Gamma_f(V_n^h,\pi_{n+1})-\Gamma_f(V_n^h,\pi_n)\le\bar g_n$ since gaps
are nonnegative, the source vanishes outside, and comparison yields
\[
 V_{n+1}^h(\tau,\cdot)\le V_n^h(\tau,\cdot)+\tau\bar g_n,
\]
so exact true-model greedy improvement is monotone.

\section{A Local Dynkin Formula up to the Exit Time}
\label{app:stopped-proof}

The proof of \cref{thm:evaluation-certificate} derives the exit-time identity
directly. We record its local form here and then give the corresponding
bound for a policy extracted from the final candidate.

\begin{lemma}
\label{lem:local-stopped-dynkin}
Let $Q$ be a bounded open box and let $\phi$ be continuous on
$[0,T]\times\overline Q$, absolutely continuous in time at each spatial
point, with bounded Borel $\partial_\tau\phi$, $\nabla_0^h\phi$, and
$\Delta_h\phi$ on $(0,T)\times Q^h$.  Let $a$ be a bounded Borel drift with
$|a_i|\le2N$, and let $Y$ be the corresponding reverse-time chain started
from $Y_0=x\in Q^h$. Write $\E_x^{\tau,a}$ for its expectation and set
\[
 \sigma_Q:=\inf\{s\ge0:Y_s\notin Q^h\}.
\]
Then for $0\le t\le\tau\le T$,
\begin{align}
 \phi(\tau,x)
 =\E_x^{\tau,a}\Bigl[&\phi\bigl(\tau-(t\wedge\sigma_Q),
 Y_{t\wedge\sigma_Q}\bigr)\nonumber\\
 &+\int_0^{t\wedge\sigma_Q}
 (\partial_\tau\phi-\cB_{a,\tau-s}^h\phi)(\tau-s,Y_s)\dd s\Bigr].
 \label{eq:local-stopped-dynkin}
\end{align}
All values of $\phi$ occurring in \eqref{eq:local-stopped-dynkin} lie in
$\overline Q$.
\end{lemma}

\begin{proof}
Stop first at $\sigma_Q\wedge t\wedge S_m$, where $S_m$ is the $m$th jump
time.  On each inter-jump interval, the map
$s\mapsto\phi(\tau-s,Y_s)$ is absolutely continuous and contributes
$-\partial_\tau\phi$. The jump increments have compensator
$(\cB_{a,\tau-s}^h\phi)(\tau-s,Y_{s-})\dd s$.
The visited states belong to the
countable lattice $x+h\mathbb Z^d$, so the time derivatives hold outside
a common null set. Subtracting the compensator therefore gives a
martingale.  Boundedness permits optional stopping.  Letting $m\to\infty$ is
legitimate because a finite-rate Poisson process has finitely many jumps on
compact intervals almost surely and all stopped variables are uniformly
bounded.  Before $\sigma_Q$, the process lies in $Q^h$, so every shifted
value required by the generator lies in $Q$.  At the exit jump,
$Y_{\sigma_Q}\in Q\setminus Q^h$ by nearest-neighbor geometry.  This proves
both the identity and the locality claim.
\end{proof}

Apply \cref{lem:local-stopped-dynkin} to $\wt V_n$ on $Q_{L'}$ and subtract
the corresponding exit-time identity for $V_n^h$.  Taking $t=\tau$ gives
\eqref{eq:stopped-evaluation-identity}; all candidate values lie in
$\overline Q_{L'}$.

The same argument controls a policy extracted from the final candidate $W$.
If a Borel feedback $\pi$ satisfies, for a constant $G\ge0$,
$0\le\Gamma_f(W,\pi)\le G$ on $(0,T)\times Q_{L'}^h$, then
\begin{equation}
 0\le V^{h,\pi}-V^{h,*}
 \le2\mathfrak C_{L'}(W)+TG
 \quad\text{on }[0,T]\times Q_L.
 \label{eq:direct-policy-certificate}
\end{equation}
Indeed, the frozen residual of $W$ differs from $\cR_{B,h}[W]$ by at most
$G$, so the
stopped evaluation estimate gives
$|W-V^{h,\pi}|\le\mathfrak C_{L'}(W)+TG$; combine this with
\cref{cor:direct-bellman} and Bellman verification.

\section{Finite-Validation Results}
\label{smsec:validation}

This section collects the random-sample and $L^2$ validation
routes summarized in the main text.

\subsection{Random Covering Radius}

Let $\mu$ be the validation distribution on a compact $D\subset\R^m$.  Assume there are $\varrho_0,c_0,C_0>0$ such that for every $0<\varrho\le \varrho_0$:
\begin{equation*}
 \mathcal N(D,\varrho/2)\le C_0\varrho^{-m},
 \qquad
 \inf_{z\in D}\mu(B(z,\varrho/2)\cap D)\ge c_0\varrho^m,
\end{equation*}
where $\mathcal N(D,\varrho/2)$ is the covering number.

\begin{theorem}
\label{thm:random-cover}
Let $Z_1,\ldots,Z_M$ be independent samples from $\mu$.  For every $0<\varrho\le \varrho_0$,
\begin{equation*}
 \Pp\bigl(\rho(\{Z_j\},D)>\varrho\bigr)
 \le C_0\varrho^{-m}\exp(-c_0M\varrho^m).
\end{equation*}
Consequently, for $0<\delta<1$, with probability at least $1-\delta$,
every $\Lambda_r$-Lipschitz field satisfies
\begin{equation*}
 \norm{r}_{\infty,D}
 \le\max_{1\le j\le M}|r(Z_j)|+\Lambda_r\varrho
\end{equation*}
for any $\varrho\le \varrho_0$ such that
$C_0\varrho^{-m}e^{-c_0M\varrho^m}\le\delta$.
\end{theorem}

\begin{proof}
Take a $\varrho/2$-cover with at most $C_0\varrho^{-m}$ centers.  If every covering ball contains a sample, then every point of $D$ lies within distance $\varrho$ of a sample.  The probability that a fixed ball is empty is at most $(1-c_0\varrho^m)^M\le e^{-c_0M\varrho^m}$.  Apply a union bound and then \cref{lem:deterministic-cover}.
\end{proof}

Up to logarithmic factors, $\varrho$ scales as $M^{-1/m}$. For a
space--time residual, $m=d+1$; for an all-control model error,
$m=d+1+\dim U$ when $U$ is Euclidean. This dimension dependence restricts
generic random covering to low-dimensional validation.

\subsection{\texorpdfstring{$L^2$}{L2} Validation with a Modulus}

A second route first obtains a population $L^2$ estimate by standard independent-sample generalization and then lifts it to the supremum norm.

\begin{lemma}
\label{lem:l2-linfty}
Let $D\subset\R^m$ be a bounded Lipschitz domain.  There exists $C_D>0$ such
that the Lipschitz representative of every $r\in W^{1,\infty}(D)$ with
$\Lip(r;D)\le\Lambda_r$ satisfies
\begin{equation}
 \norm{r}_{\infty,D}
 \le C_D\Lambda_r^{m/(m+2)}\norm{r}_{L^2(D)}^{2/(m+2)}
 +C_D\norm{r}_{L^2(D)}.
 \label{eq:l2-linfty}
\end{equation}
\end{lemma}

The proof is given in \cref{app:validation-proof}.  The exponent $2/(m+2)$ is again dimension dependent, and neither \cref{thm:random-cover} nor \cref{lem:l2-linfty} applies to a discontinuous residual without additional structure.

\subsection{Switching Policies}

If $\pi_n$ jumps across a switching surface, the frozen residual may fail to be globally Lipschitz even when the value network and learned dynamics are smooth.  Three rigorous options are available: validate separately on policy cells and control the interfaces; use a deterministic net together with a verified local modulus on each cell; or replace the discontinuous selector by a regularized policy and include the induced greedy gap.  Without one of these devices, finite-sample statistics are proxies for, not upper bounds on, the suprema required by the theorems.

\section{Proof of the Validation Lifting Inequality}
\label{app:validation-proof}

\begin{proof}[Proof of \cref{lem:l2-linfty}]
A bounded Lipschitz domain has a uniform interior density property: there are
$r_D,c_D>0$ such that
\begin{equation*}
 |B_\rho(z)\cap D|\ge c_D\rho^m,
 \qquad z\in\overline D,\quad0<\rho\le r_D.
\end{equation*}
Let $M=\norm{r}_{\infty,D}$.  The case $\Lambda_r=0$ follows immediately
from the $L^2$ norm on each connected component, so assume $\Lambda_r>0$.
Choose $z_0\in\overline D$ with $|r(z_0)|\ge M/2$ and set
\[
 \rho=\min\left\{\frac{M}{4\Lambda_r},r_D\right\}.
\]
For $z\in B_\rho(z_0)\cap D$,
$|r(z)|\ge M/2-\Lambda_r\rho\ge M/4$.  Hence
\begin{equation}
 \norm{r}_{L^2(D)}^2
 \ge \frac{M^2}{16}c_D\rho^m.
 \label{eq:lifting-core}
\end{equation}
If $\rho=M/(4\Lambda_r)$, rearranging \eqref{eq:lifting-core} gives
\[
 M\le C_D\Lambda_r^{m/(m+2)}
\norm{r}_{L^2(D)}^{2/(m+2)}.
\]
If $\rho=r_D$, then \eqref{eq:lifting-core} gives
$M\le C_D\norm{r}_{L^2(D)}$.  Combining the two alternatives proves
\eqref{eq:l2-linfty}.
\end{proof}

\section{Experiment Configurations and Solver Details}
\label{app:reproducibility}

All experiments run on one NVIDIA GeForce RTX~5070 (12\,GB) with
PyTorch~2.11.0+cu128.  Grid references, Riccati flows, dynamics identification,
and stated stencil audits use double precision; neural optimization uses
single-precision tensors.  The active-box and sparse-actuator runs enable
TensorFloat-32 (TF32) matrix multiplication.  Random seeds are fixed in the
stand-alone scripts,
and numerical outputs are stored in machine-readable
result files.  All scripts and result files are publicly available at
\url{https://github.com/yeoneung/semi_discrete_pinn}.
Unless stated otherwise, neural summaries use three independent training
seeds.

\paragraph{Grid references}
Except for the deliberately nonmonotone rows of the E4 envelope test,
semi-discrete Bellman and frozen-policy equations are integrated by monotone
explicit Euler in the rate form
\eqref{eq:generator-rates}, with time step
$\Delta\tau=0.45/\Lambda_h=0.45\,h/(2dN)$ except in the consistency study
specified below.  Because the last time step may be shortened, the number of
jumps in one coordinate is a sum of independent Bernoulli variables, with
mean $\theta=2NT/h$, whose moment generating function is bounded by that of
$\Pois(\theta)$.  Hence, for a coordinate margin $D>0$ and
$m=\lceil D/h\rceil\ge\theta$, its exit tail is at most
$d(e\theta/m)^m$.  Replication or exact-solution padding therefore changes an
interior value only by the corresponding payoff range times this tail.
Reference solves use up to $2001^2$ nodes ($d=2$,
$h=0.0025$); the finest solve used in the neural refinement study completes
in $26$ seconds on the GPU.

\begin{table}[tbhp]
\centering
\scriptsize
\setlength{\tabcolsep}{2.5pt}
\caption{Configurations of the grid and chain experiments.  Here $Q$ is
the computational box and $K$ the target box.  A displayed
Courant--Friedrichs--Lewy (CFL) step is
reduced slightly when needed so that the final step lands at $T$.}
\label{smtab:grid-config}
\begin{tabularx}{\textwidth}{@{}lccclYYc@{}}
\toprule
Experiment & $d$ & $T$ & $N$ & $h$ & $Q$ & $K$ & $\Delta\tau$\\
\midrule
E1 consistency & 1 & $0.5$ & $0.5$ & $0.04$--$0.0025$
& $(-3,3)$ & $(-1,1)$ & $0.05625/\Lambda_h$\\
E1 consistency & 2 & $0.5$ & $0.5$ & $0.04$--$0.005$
& $(-2,2)^2$ & $(-0.75,0.75)^2$ & $0.05625/\Lambda_h$\\
E2 exact PI & 1 & $1$ & $0.5$ & $0.02$--$0.0025$
& $(-3,3)$ & $(-1,1)$ & $0.45/\Lambda_h$\\
E3 chain tail & 1 & $1$ & $0.5$ & $0.01$
& --- & --- & continuous time\\
E3 localization & 1 & $1$ & $0.5,1$ & $0.01$
& $(-5,5)$ & $(-1,1)$ & $0.45/\Lambda_h$\\
E4 envelope & 1 & $0.5$ & $0$--$4$ & $0.01$
& $(-3,3)$ & $(-1,1)$ & $0.45h/(2\max\{N,0.5\})$\\
\bottomrule
\end{tabularx}
\end{table}

\paragraph{Riccati reference}
The linear--quadratic regulator (LQR) calibration uses
$\dot x=Ax+u$, $L=(|x|^2+|u|^2)/2$, $g=|x|^2/2$, $T=1$, and
$U=[-1.5,1.5]^d$, where
$(Ax)_i=-x_i+0.2x_{i-1}+0.2x_{i+1}$ with periodic indices.  Here
$h=0.01$ and $N=1.45$.  For the unconstrained reference,
$V^h=x^\top P(\tau)x/2+c(\tau)$, where
$P'=A^\top P+PA+I-P^2$, $P(0)=I$, and
$c'=Nh\operatorname{tr}P$, $c(0)=0$.  The Fourier-mode Riccati flow is
integrated in double precision by the classical fourth-order Runge--Kutta
method with $4000$ steps, and $c$ by the trapezoidal rule.  No exact or
learned feedback is clipped on the evaluation samples at
$\tau=0.25,0.5,0.75,1$; at $\tau=1$, the analytic row-sum box bound is
$0.543<u_{\max}=1.5$.  Thus the Riccati solution is used as a local
reference for the compact-control equation, not as a certified whole-space
compact-control value.

\paragraph{High-dimensional tests}
The active-box benchmark uses $U=[-0.5,0.5]^d$, $h=0.01$, $N=0.75$, and
$T=0.75$.  The shared on-site/pair architecture has $3783$ parameters at
every dimension.  Checkpoints are selected only by a fixed held-out
frozen-policy residual.  The sparse Allen--Cahn test uses $h=0.005$,
$N=5.5$, $T=0.75$, $|u_i|\le3$, and actuators $i=1,9,17,\ldots$.
Proportional gains and separately trained linear and nonlinear local policies
are evaluated on the same ten independently generated state ensembles.

\paragraph{Neural configurations}
All value models impose the initial condition exactly as
$\wt V_\theta(\tau,x)=g(x)+\tau F_\theta(\tau,x)$, so $r_n\equiv0$;
$F_\theta$ is either a multilayer perceptron (MLP) or the stated structured
model.  Activations are sigmoid linear units (SiLU), and optimization uses
Adam or AdamW with cosine learning-rate decay.
\Cref{tab:app-config} lists the per-experiment constants.  Generic shifted
queries are coordinate streamed; the finite-range models instead assemble
their exact local stencil in $O(RBd)$ work for interaction radius $R$.

\begin{table}[tbhp]
\centering
\scriptsize
\setlength{\tabcolsep}{2.5pt}
\caption{Configurations of the neural experiments.  ``PI'' denotes the
number of outer policy iterations; steps are Adam steps per iteration.}
\label{tab:app-config}
\begin{tabularx}{\textwidth}{@{}>{\raggedright\arraybackslash}XYYYYYYY@{}}
\toprule
Experiment & $d$ & \makecell{$h$\\$\nu_h$} & $N$ & widths & batch &
\makecell{PI\\$\times$steps} & \makecell{learning\\rate}\\
\midrule
Estimator & 1 & \makecell{$0.01$\\$0.005$} & $0.5$ &
$(64)^3$ & 4096 & $1\times13500$ & $2\cdot10^{-3}$\\
Structured certificate & 1 & \makecell{$0.1$\\$0.1$} & $1$ &
time net, width 8 & 4096 & $1\times5000$ & $5\cdot10^{-3}$\\
Duffing & 2 & \makecell{$0.005$\\$0.0105$} & $2.1$ &
$(64)^3$ & 8192 & $6\times3000$ & $2\cdot10^{-3}$\\
Duffing $h$-study & 2 & \makecell{$0.02$--$0.005$\\$2.1h$} & $2.1$ &
$(64)^3$ & 8192 & $6\times3000$ & $2\cdot10^{-3}$\\
Spurious study & 1 & \makecell{$0.01$\\$0.005$} & $0.5$ &
$(64)^3$ & 4096 & $1\times6000$ & $2\cdot10^{-3}$\\
Three-way, smooth & 2 & \makecell{$0.005$\\$0.0105$} & $2.1$ &
$(64)^3$ & 8192 & $1\times18000$ & $2\cdot10^{-3}$\\
Switching Duffing & 2 & \makecell{$0.005$\\$0.0105$} & $2.1$ &
$(64)^3$ & 8192 & $2\times3000$ & $2\cdot10^{-3}$\\
Adaptive eikonal & 1 & \makecell{$0.01$\\$0.005$} & $0.5$ &
$(64)^3$ & 4096 & $1\times6000$ & $2\cdot10^{-3}$\\
Learned dynamics & 2 & \makecell{$0.005$\\$0.0105$} & $2.1$ &
cubic library $+(64)^3$ & 8192 & $4\times3000$ & $2\cdot10^{-3}$\\
Spectral LQR & 64--1024 & \makecell{$0.01$\\$0.0145$} & $1.45$ &
deg.-4 spectral quad. & 256 & $4\times3000$ & $2\cdot10^{-3}$\\
Active box & 64--1024 & \makecell{$0.01$\\$0.0075$} & $0.75$ &
local, width 40 & 128--160 & $3\times600$ & $10^{-3}$\\
Sparse Allen--Cahn & 256, 1024 & \makecell{$0.005$\\$0.0275$} & $5.5$ &
range $0$--$4$, width 40 & 128--160 & $(3$--$4)\times(600$--$700)$ & $10^{-3}$\\
\bottomrule
\end{tabularx}
\end{table}

\paragraph{Sampling and selection}
The estimator uses $Q_{L'}=(-3,3)$ and target $Q_L=(-1,1)$ under the fixed
feedback stated in the main text.  The Duffing runs sample
$Q_{L'}=(-1.5,1.5)^2$, use $K=(-1,1)^2$ for errors, start from the zero
policy, and warm-start each subsequent value network; the reported iterate is
the final one.  The switching-control comparison instead starts from the
terminal-greedy policy.  The eikonal runs sample $(-2,2)$ and retain the final iterate.
The spectral LQR samples $[-0.9,0.9]^d$.  The active-box training mixture and
fixed validation set lie in $[-\pi,\pi]^d$; it starts from the zero policy and
selects checkpoints by held-out frozen-policy residual.  Sparse Allen--Cahn
samples its residual on $[-0.9,0.9]^d$ together with structured states, starts
from the validation-selected proportional feedback, and selects value
checkpoints by held-out residual.  Error normalizations are defined with the
corresponding result tables below.

\paragraph{Validation protocol}
Residual maxima are reported over deterministic validation nets
(\cref{sec:exp-certificate}: $201\times1201$ with a Lipschitz correction
estimated from independent samples) or over independent random samples of
size $4\cdot10^4$.  These quantities are labeled empirical because the
network moduli are not independently verified upper bounds.  Learned-model
error maxima use a $201\times201\times21$ state--control grid and are likewise
finite-grid diagnostics.  Monte Carlo tails use $2\times10^6$ paths with the
same marked-Poisson construction as the proofs.

\clearpage
\section{Additional Numerical Diagnostics}
\label{smsec:additional-experiments}

This section records the plots and tables summarized in
\cref{sec:experiments}.  All values are read from the machine-readable
result files distributed with the stand-alone scripts.

\subsection{Consistency Against an Exact Viscosity Solution}

For the eikonal benchmark of \cref{sec:exp-consistency}, the reported error
is the final-time pointwise supremum
\[
 \norm{V^{h,*}(T,\cdot)-V(T,\cdot)}_{\infty,K}.
\]
In \cref{smfig:e1-consistency}, the last three one-dimensional refinements have rates
$0.545$, $0.534$, and $0.525$; the two-dimensional rates are $0.616$,
$0.618$, and $0.603$.  Thus the observed behavior is close to $\sqrt h$.
Here $\Delta\tau=0.05625/\Lambda_h$.  Halving this step on the finest meshes
changes the solution by $2.03\times10^{-4}$ in $d=1$ and
$1.33\times10^{-4}$ in $d=2$, respectively $1.4\%$ and $0.5\%$ of the
reported errors; the audit is stored in the same machine-readable result
file.
The benchmark uses a nonunique minimizer at $p=0$ and the sharp envelope
$N=1/2$; the cited estimate \eqref{eq:sqrt-h-consistency} assumes a unique
Lipschitz selector and $N\ge1$.

\begin{figure}[tbhp]
\centering
\includegraphics[width=0.52\textwidth]{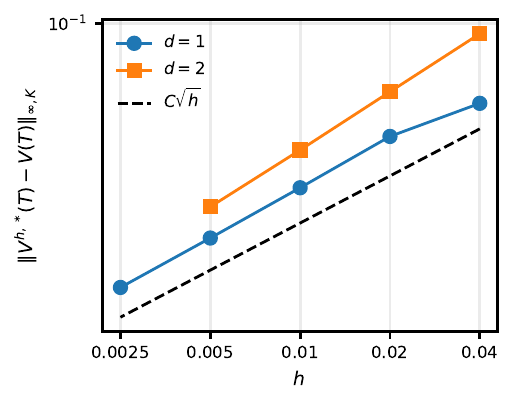}
\caption{Final-time pointwise-supremum error of the semi-discrete Bellman
solution against the exact eikonal viscosity solution: $d=1$
($h=0.04,\ldots,0.0025$) and $d=2$ ($h=0.04,\ldots,0.005$), at
$N=M_{\comp}/2$.}
\label{smfig:e1-consistency}
\end{figure}

\subsection{Exact Policy Iteration}

For the saturating test described in the main text, exact policy iteration
starts from $\pi_0\equiv0$.  We use
$h=0.02,0.01,0.005,0.0025$.  The errors are nearly independent of
$h$ and decay approximately quadratically until roundoff; see
\cref{smfig:e2-pi}.  Write
$e_{n,K}(s):=\sup_{x\in K}(V_n^h-V^{h,*})(s,x)$.  The empirical
unrestricted slice-gradient ratio on $K$ scales as $h^{-1}$.  Restricting to slices with
$e_{n,K}(\tau)\ge0.01\sup_s e_{n,K}(s)$ gives values in $[2,15]$ for $n\le3$, with less
than $6\%$ variation over the tested range.  To retain all time slices, set
\begin{equation*}
 \widehat C_{n,h}^{\rm int}:=
 \frac{\displaystyle\int_0^T\sup_{x\in K}
 |\nabla_0^h(V_n^h-V^{h,*})(s,x)|\dd s}
 {\displaystyle\int_0^T e_{n,K}(s)\dd s}.
\end{equation*}
For $n=0,1,2,3$, respectively, its ranges over the four meshes are
$1.99$--$2.00$, $2.87$--$2.94$, $4.28$--$4.43$, and $7.57$--$9.10$.
The full-horizon average varies moderately over the tested refinements.
These diagnostics average or restrict time on $K$, whereas
\eqref{eq:slice-pi-gradient} requires a global time-slice bound.

\begin{figure}[tbhp]
\centering
\includegraphics[width=0.95\textwidth]{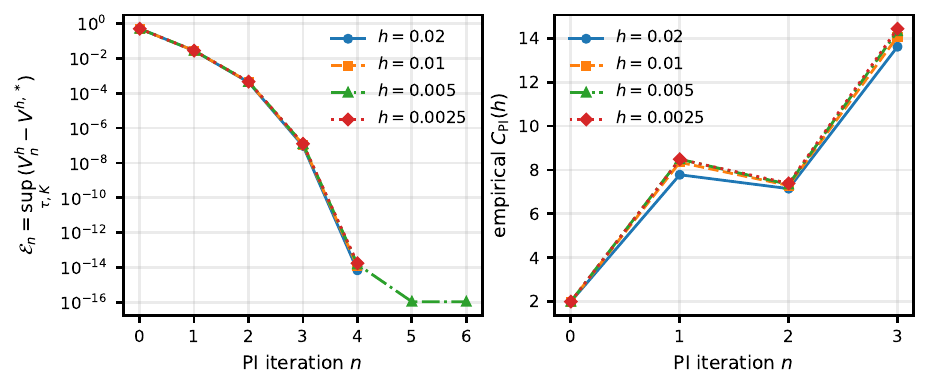}
\caption{Exact fixed-$h$ policy iteration.  Left: $\cE_n$ for four values
of $h$.  Right: empirical slice-gradient ratio on slices carrying at
least $1\%$ of the peak error.}
\label{smfig:e2-pi}
\end{figure}

\subsection{Estimator Effectivity}

\Cref{smtab:e5-certificate} gives the three-seed checkpoints summarized in
\cref{sec:exp-certificate}.  The initial mismatch is zero and the strip
term is below $10^{-16}$, so the indicator is the empirically lifted
residual maximum multiplied by $T=1$.  It exceeds the measured error in all
$12$ seed--checkpoint pairs; the smallest observed slack is $2.34\cdot10^{-3}$.

\begin{table}[tbhp]
\centering
\small
\caption{Empirical residual indicator versus time-sampled reference-grid
error ($d=1$, $h=0.01$).  Entries are mean $\pm$ sample standard deviation
over three seeds; effectivity is their ratio.}
\label{smtab:e5-certificate}
\begin{tabular}{lccc}
\toprule
Adam steps & indicator & measured error & effectivity\\
\midrule
$500$ & $0.122\pm0.020$ & $0.0393\pm0.0083$ & $3.26\pm1.16$\\
$1500$ & $0.0467\pm0.0028$ & $0.0298\pm0.0028$ & $1.57\pm0.06$\\
$4500$ & $0.0352\pm0.0015$ & $0.0291\pm0.0005$ & $1.21\pm0.03$\\
$13500$ & $0.0238\pm0.0016$ & $0.0194\pm0.0041$ & $1.26\pm0.24$\\
\bottomrule
\end{tabular}
\end{table}

The reference maximum is taken over ten time slices and hence lower-bounds
the full space--time error.  Together with the sample-based Lipschitz
correction, this makes the reported effectivity an empirical diagnostic, not
a rigorous certificate.

For comparison, the direct certificate of \cref{cor:direct-bellman} is
verified on the one-dimensional problem

\[
 U=[-1,1],\quad f(u)=u,\quad h=0.1,\quad N=T=1,
 \quad g(x)=\cos x,
\]

with running cost
$L(\tau,x,u)=1+\tfrac12(\sin h/h)^2e^{-2\gamma\tau}\sin^2x+u^2/2$,
where $\gamma=2(1-\cos h)/h$.  Its exact semi-discrete value is
$V^{h,*}(\tau,x)=e^{-\gamma\tau}\cos x+\tau$.  A $25$-parameter candidate
$W_\theta(\tau,x)=A_\theta(\tau)\cos x+\tau$ is trained from Bellman-residual
samples, without value targets. The prescribed spatial factor permits
analytic verification. On
$Q=(-8,8)$ and $K=(-1,1)$, spatial
dependence of the residual is bounded analytically.  In time we use $20000$
cells, outward-rounded interval evaluation at their endpoints, and the
verified global Lipschitz bound $0.43950$.  The resulting terms are

\[
 \rho_B\le 5.802\mathbin{\cdot}10^{-5}
       +1.099\mathbin{\cdot}10^{-5}
       =6.901\mathbin{\cdot}10^{-5},\qquad
 (A_Q+M_V)\Psi_h\le4.98\mathbin{\cdot}10^{-17}.
\]

Here $A_Q\le2.00001$, $M_V\le2.99834$, and
$\Psi_h\le9.955\cdot10^{-18}$.  The initial mismatch is zero, so the
certificate is $6.901\cdot10^{-5}$.
The same interval procedure gives the independent true-error bound
$9.000\cdot10^{-6}$ (the dense-grid value is $1.687\cdot10^{-6}$).
All certificate terms and the verified true-error bound use $80$-digit Arb
ball arithmetic through \texttt{python-flint} 0.9.0, with the stored
binary64 weights interpreted exactly; no sampled maximum enters the stated
certificate.  The analytic control bound is
$0.99834<1$, so the interior formula used in the residual is valid.

\subsection{Continuous and Shifted Residuals}

The eikonal test of \cref{sec:exp-pinn} uses identical networks,
optimizers, and budgets for continuous and shifted residuals.  Continuous
runs converge to the spurious almost-everywhere solution under all three
sampling schemes.  The shifted residual exposes the stencil-band defect,
but reaches the lower-error solution only when the band receives $O(1)$
sampling mass; see \cref{smtab:e6b-pinn}.

\begin{table}[tbhp]
\centering
\footnotesize
\setlength{\tabcolsep}{4pt}
\caption{Continuous versus shifted residual on the eikonal problem
($h=0.01$, three seeds; ranges over seeds).  Both residual maxima are
evaluated on the same validation net.}
\label{smtab:e6b-pinn}
\begin{tabular}{llcccc}
\toprule
residual & sampling & pointwise-sup error & final loss &
cont.\ max & shifted max\\
\midrule
continuous & uniform / kink / band & $0.5000$ & $3\cdot10^{-8}$
& $0.001$ & $2.00$\\
shifted & uniform & $0.4933$ & $6\cdot10^{-3}$ & $0.014$ & $1.99$\\
shifted & kink concentrated & $0.41$--$0.44$ & $3.5\cdot10^{-2}$
& $0.14$--$0.22$ & $1.90$\\
shifted & $30\%$ on $|x|\le3h$ & $0.070$--$0.096$
& $3\cdot10^{-3}$ & $1.04$--$1.14$ & $0.71$--$0.86$\\
\bottomrule
\end{tabular}
\end{table}

On smooth Duffing, the continuous-residual target is compared with an
$h=0.0025$ grid proxy and gives error $2.0$--$2.3\times10^{-2}$; it is not
ranked against the shifted methods because their reference is the
$h=0.005$ semi-discrete solution.  For the matched comparison at $h=0.005$,
direct shifted training gives $0.00688\pm0.00158$ in $70.3\pm0.6$ seconds,
whereas SDPI gives $0.01029\pm0.00150$ in $73.9\pm0.4$ seconds.  Both use
the same network, box, batch size, optimizer family, $N$, and $18000$ Adam
steps; SDPI's time includes all six policy updates.  The two grid references
at $h=0.005$ and $0.0025$ differ by $1.79\times10^{-2}$.

For the switching-control comparison in \cref{smfig:e6-matched}, all settings are unchanged except that
the control penalty is removed, so
$\min_{|u|\le1}uD_2^{0,h}V=-|D_2^{0,h}V|$.  Seed $0$ fixes the
terminal-greedy initialization $\pi_0(x)=-\operatorname{sign}(x_2)$ and
$6000$-step budget; the reported five
seeds are held out from that choice.  SDPI gives $0.02209\pm0.00445$ versus
$0.03029\pm0.00504$ for direct training and is lower in all five pairs; the
mean paired ratio is $0.728$, with bootstrap $95\%$ interval
$[0.687,0.772]$.  The corresponding times are $43.0\pm0.4$ and
$37.7\pm0.1$ seconds.  Giving direct training $9000$ steps reduces its mean
error to $0.02440\pm0.00274$, still above SDPI in four pairs.  Halving the
reference time step changes the value on $K$ by at most $7.97\cdot10^{-5}$;
enlarging the grid box from $(-2.5,2.5)^2$ to $(-3,3)^2$ changes it by
$6.7\cdot10^{-16}$.

\begin{figure}[tbhp]
\centering
\includegraphics[width=0.96\textwidth]{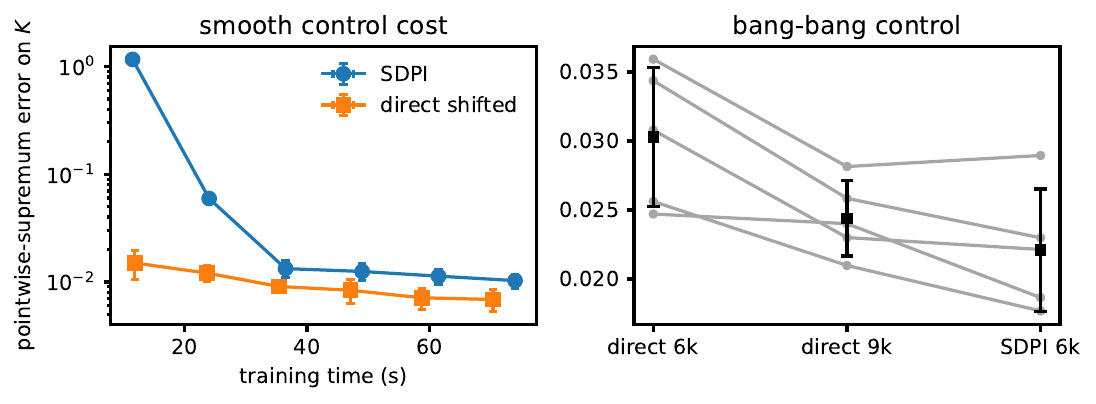}
\caption{Matched Duffing comparisons against the common $h=0.005$
semi-discrete reference.  Left: on the smooth quadratic-control problem,
direct training is more accurate through $18000$ steps.  Right: on the
bang--bang variant, SDPI has lower error at $6000$ steps on all five held-out
seeds; direct is also shown with $9000$ steps.  Black squares and bars show
the mean and sample standard deviation.}
\label{smfig:e6-matched}
\end{figure}

\paragraph{Residual-adaptive refinement}
In \cref{smfig:e6e-adaptive}, each refresh ranks $65536$ candidates by the
relevant residual and retains $512$ points for replay. With the same $6000$ Adam steps,
batch size, and $30\%$ replay fraction, the three-seed errors are
$0.04398\pm0.00260$ for shifted-residual adaptation and
$0.50000\pm0.00004$ for continuous-residual adaptation.  The first refresh
places $0.768\pm0.092$ and $0.0039\pm0.0068$, respectively, of the selected
points in $|x|\le3h$.  The comparison equalizes collocation budgets, not wall
time.  A six-seed shifted-residual continuation over
$h=0.02,0.01,0.005$ gives mean errors $0.0504$, $0.0327$, and $0.0347$; the
last refinement reaches a plateau.

\begin{figure}[tbhp]
\centering
\includegraphics[width=0.92\textwidth]{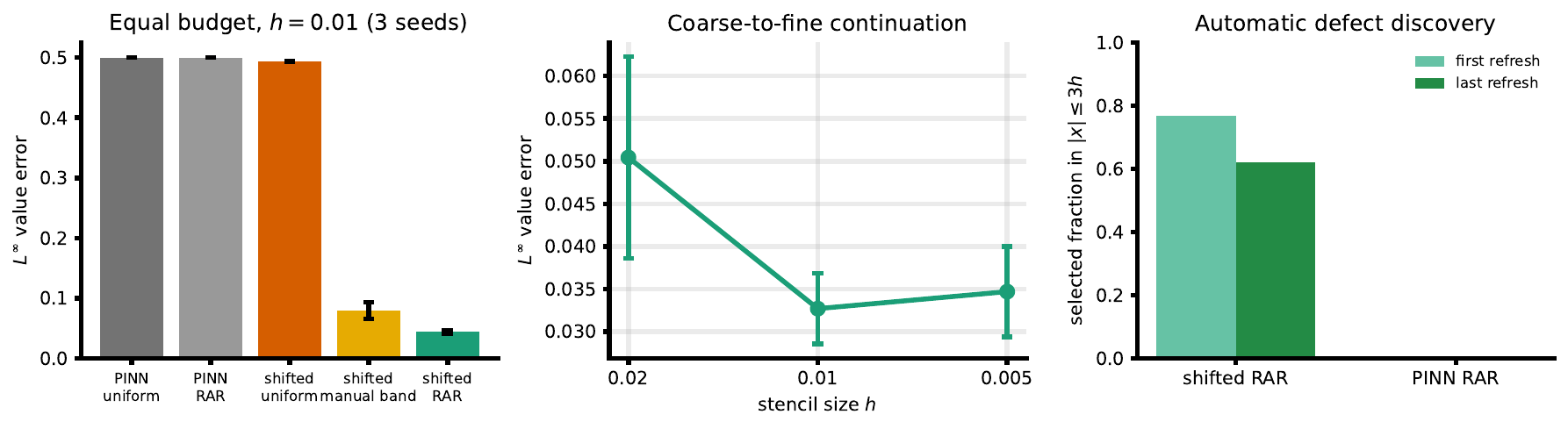}
\caption{Residual-adaptive eikonal experiment.  The shifted residual
identifies the stencil-scale defect and reduces the value error; the
continuous residual does not select the kink. Error bars show sample
standard deviations.}
\label{smfig:e6e-adaptive}
\end{figure}

\FloatBarrier
\subsection{Neural Refinement on the Duffing Problem}

The neural refinement study uses the same architecture, optimizer, six
policy iterations, and $3000$ steps per iteration as the Duffing experiment
in the main text.  Errors are maxima over $K=(-1,1)^2$ at
$\tau=0.25,0.5,0.75,1$; neural entries are ranges over three independent
training seeds.  The grid-bias proxy compares nested grid values with the
$h_{\mathrm{ref}}=0.0025$ monotone solution.

\begin{table}[tbhp]
\centering
\small
\caption{Neural $h$-refinement on Duffing.  ``Own'' denotes the monotone
reference at the training mesh and ``fine'' the $h_{\mathrm{ref}}=0.0025$ proxy.}
\label{smtab:duffing-h-refinement}
\begin{tabular}{cccc}
\toprule
$h$ & neural--own error & grid-bias proxy & neural--fine error\\
\midrule
$0.02$  & $0.0047$--$0.0071$ & $0.1240$ & $0.1276$--$0.1282$\\
$0.01$  & $0.0067$--$0.0094$ & $0.0539$ & $0.0594$--$0.0614$\\
$0.005$ & $0.0089$--$0.0119$ & $0.0179$ & $0.0255$--$0.0266$\\
\bottomrule
\end{tabular}
\end{table}

\subsection{Learned-Dynamics Diagnostics}

We sample $x$ in $(-1.5,1.5)^2$ and $u$ in $[-1,1]$, add independent Gaussian
noise with standard deviation $0.02$ to the drift labels, and fit by
ridge least squares using all monomials in $(x_1,x_2,u)$ of total degree at
most three.  Thus the library contains the Duffing drift without revealing
which coefficients are nonzero.  The finite-grid envelope is evaluated on a
$201\times201\times21$ state--control grid and is the maximum Euclidean
drift error over that grid and all tested controls.  Five identification
seeds are used at each of $M=100,300,1000,3000,10000$; its mean decays as
$M^{-0.605}$.

\begin{table}[tbhp]
\centering
\scriptsize
\setlength{\tabcolsep}{2.5pt}
\caption{Learned Duffing dynamics and end-to-end SDPI.  The finite-grid
envelope uses five identification seeds; value error and transfer term use
three end-to-end seeds.  Entries are mean $\pm$ sample standard deviation.}
\label{smtab:e7c-learned}
\begin{tabular}{rccccc}
\toprule
$M$ & finite-grid $\bar\eta$ & value error & excess over known & model term
& margin pass\\
\midrule
$10^2$ & $0.0992\pm0.0203$ & $0.0342\pm0.0079$ & $0.0243$
& $0.1099\pm0.0300$ & $0/5$\\
$10^3$ & $0.0241\pm0.0030$ & $0.0144\pm0.0020$ & $0.00450$
& $0.0282\pm0.0016$ & $0/5$\\
$10^4$ & $0.00538\pm0.00071$ & $0.0131\pm0.0015$ & $0.00319$
& $0.00575\pm0.00110$ & $5/5$\\
\bottomrule
\end{tabular}
\end{table}

The known-dynamics error is $0.00992$.  The model term evaluated here is
\[
 T\bar\eta\,\esssup_{\tau\in(0,T)}
 \norm{\nabla_0^h\wt V_n(\tau,\cdot)}_{\infty,Q_{L'}^h},
\]
using the finite-grid envelope and a separate candidate sample.  ``Margin pass'' means
that the measured componentwise model error leaves the robust monotonicity
margin in \eqref{eq:robust-monotonicity} positive.  In each of the nine end-to-end
runs, the measured model-error term exceeds the excess over the baseline,
providing a finite-grid audit of the dependence in
\cref{thm:evaluation-certificate}; see \cref{smfig:e7c-learned}.

\begin{figure}[tbhp]
\centering
\includegraphics[width=0.92\textwidth]{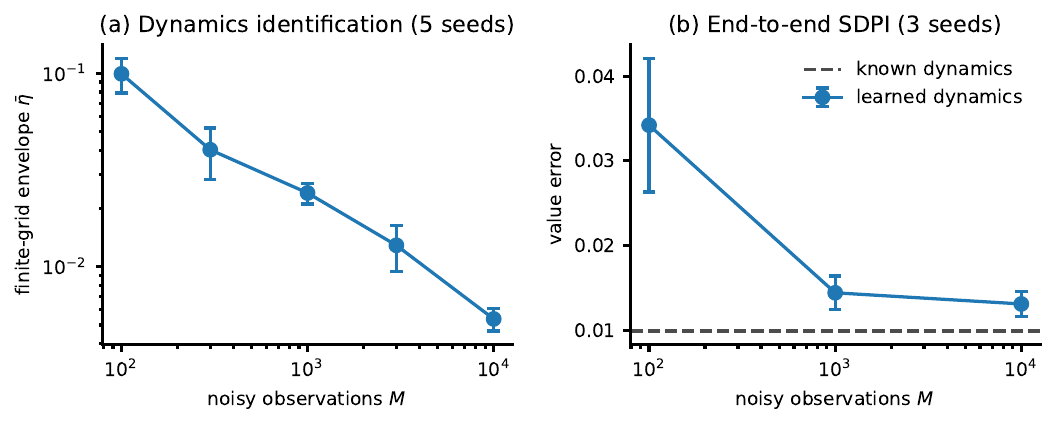}
\caption{Noisy structured dynamics identification and the resulting SDPI
error.  Error bars summarize the independent seeds described above.}
\label{smfig:e7c-learned}
\end{figure}

\subsection{Structured LQR Calibration}

The architecture calibration in \cref{smfig:e8c-lqr} uses the periodic coupled system
$(Ax)_i=-x_i+0.2x_{i-1}+0.2x_{i+1}$ with $|u_i|\le1.5$.  Its Fourier modes
diagonalize the Riccati flow, and centered differences are exact on the
quadratic value.  The spectral model learns a shared degree-four function of
time and the graph eigenvalue.  It is trained only through the
frozen-policy shifted residual; the Riccati solution is used for evaluation.

\begin{table}[tbhp]
\centering
\scriptsize
\setlength{\tabcolsep}{2pt}
\caption{Coupled LQR calibration.  Spectral entries are ranges over three
seeds; generic-MLP entries are single runs used as an
architecture diagnostic.}
\label{smtab:e8c-lqr}
\begin{tabular}{llcccc}
\toprule
model & $d$ & relative $L^2$ & relative sampled max & policy rel. $L^2$
& cost ratio\\
\midrule
spectral & 64 & $(3.18\text{--}5.09)\cdot10^{-4}$ & $(1.39\text{--}2.62)\cdot10^{-3}$
& $(1.71\text{--}2.89)\cdot10^{-3}$ & $1.00000$\\
spectral & 256 & $(7.39\text{--}8.35)\cdot10^{-4}$ & $(3.43\text{--}4.66)\cdot10^{-3}$
& $(8.39\text{--}9.74)\cdot10^{-3}$ & $1.00002$\\
spectral & 1024 & $(1.13\text{--}1.57)\cdot10^{-3}$ & $(3.05\text{--}4.28)\cdot10^{-3}$
& $(1.48\text{--}1.86)\cdot10^{-2}$ & $1.00004$--$1.00005$\\
generic MLP & 64 & $8.92\cdot10^{-3}$ & $3.39\cdot10^{-2}$ & $7.07\cdot10^{-2}$ & $1.001$\\
generic MLP & 256 & $6.76\cdot10^{-2}$ & $0.319$ & $0.836$ & $1.239$\\
generic MLP & 1024 & $0.332$ & $1.03$ & $2.25$ & $10.59$\\
\bottomrule
\end{tabular}
\end{table}

For $h=0.02,0.01,0.005$, the three-seed mean neural errors against the
corresponding semi-discrete values are $(7.42,7.98,8.21)\cdot10^{-4}$,
whereas the viscosity biases are $0.170$, $0.0852$,
and $0.0426$.  The total-error rates are $0.999$ and $0.999$.  This confirms
the expected bias separation for the structured quadratic family. The
audited controls are unsaturated; the following nonlinear benchmark tests
active compact-control constraints.

\begin{figure}[tbhp]
\centering
\includegraphics[width=0.92\textwidth]{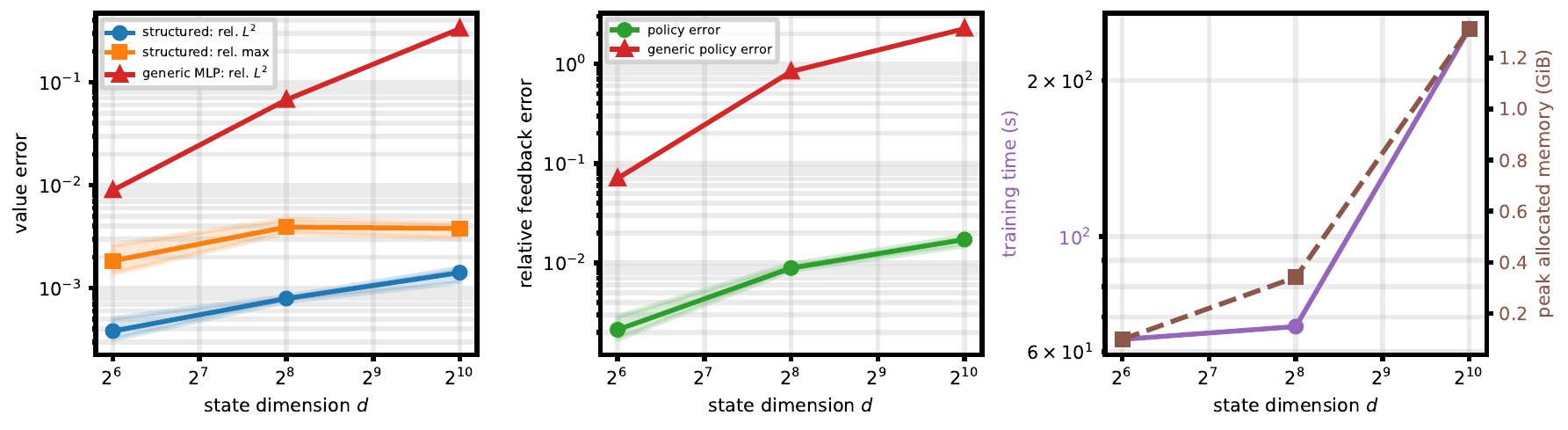}
\caption{Structured LQR calibration through $d=1024$: accuracy of the
spectral model and deterioration of a generic MLP.}
\label{smfig:e8c-lqr}
\end{figure}

\subsection{Nonlinear Active-Box Exact Reference}

For the problem in \cref{sec:exp-scaling}, write
$\operatorname{sat}_{0.5}(z)=\max\{-0.5,\min\{z,0.5\}\}$.  The
componentwise minimizer is
$u_i^*=\operatorname{sat}_{0.5}(-D_i^{0,h}V^{h,*})$.  The manufactured
state cost is chosen by substituting the displayed $V^{h,*}$ into the
semi-discrete Bellman equation and is globally nonnegative; its analytic
lower bound is $0.2025d$.  The exact construction supplies the running cost,
but value and policy targets are not used in training.  Float64 audits give
maximum discrepancies $1.61\cdot10^{-13}$ for the centered gradient,
$4.06\cdot10^{-10}$ for the discrete Laplacian, and zero for the assembled
Bellman identity.

The value model is
\[
 \begin{aligned}
 \wt V_\theta=g+\tau\biggl[{}
 &\sum_{i=1}^d\phi_\theta(\tau/T,\sin x_i,\cos x_i)\\
 &+\sum_{i=1}^{d-1}\psi_\theta\bigl(\tau/T,\sin(x_i-x_{i+1}),
       \cos(x_i-x_{i+1}),1\bigr)+d c_\theta(\tau/T)\biggr].
 \end{aligned}
\]
The shared on-site and pair networks have two hidden layers of width $40$;
$c_\theta$ has one hidden layer of width $20$.  The complete terminal cost
$g=E$, including its pair term, is hard-coded in every run.  Thus the
radius-zero ablation removes the learned pair correction $\psi_\theta$, not
the pair term in $g$.

\begin{table}[tbhp]
\centering
\scriptsize
\setlength{\tabcolsep}{2.5pt}
\caption{Active-box benchmark.  Entries are mean $\pm$ sample standard
deviation over three training seeds; active fraction is fixed by the common
test set.  The first value error is normalized by $\norm{V^{h,*}-g}_{L^2}$;
the second removes the known offset $1.5d\tau$ from both values before
normalization.}
\label{smtab:e11-active}
\begin{tabular}{rcccc}
\toprule
$d$ & correction rel. $L^2$ & offset-free rel. $L^2$ & policy rel. $L^2$
& active fraction / accuracy\\
\midrule
64 & $(6.38\pm3.73)\cdot10^{-4}$ & $(3.69\pm2.16)\cdot10^{-4}$
& $(1.05\pm0.47)\cdot10^{-2}$ & $0.6840$ / $0.9955\pm0.0022$\\
256 & $(4.84\pm1.94)\cdot10^{-4}$ & $(2.80\pm1.12)\cdot10^{-4}$
& $(1.46\pm0.63)\cdot10^{-2}$ & $0.6848$ / $0.9944\pm0.0024$\\
1024 & $(6.15\pm2.49)\cdot10^{-4}$ & $(3.60\pm1.46)\cdot10^{-4}$
& $(2.82\pm1.03)\cdot10^{-2}$ & $0.6853$ / $0.9889\pm0.0041$\\
\bottomrule
\end{tabular}
\end{table}
\FloatBarrier

At $d=256$, removing the learned nearest-neighbor correction gives correction
error $0.29012\pm0.00012$, policy error $0.31224\pm0.00046$, and active-set
accuracy $0.84310\pm0.00021$; restoring it yields the middle row of
\cref{smtab:e11-active}.  The local interaction is therefore essential to
the high-dimensional result.

A matched $d=1024$ comparison uses the same architecture, batch, training and
test samples, AdamW optimizer, and three $600$-step cosine cycles.  Direct
Bellman-residual training gives correction error
$(5.55\pm1.21)\cdot10^{-4}$, policy error
$(2.72\pm0.40)\cdot10^{-2}$, and time $26.81\pm0.01$ seconds.  The SDPI
values are $(6.15\pm2.49)\cdot10^{-4}$,
$(2.82\pm1.03)\cdot10^{-2}$, and $30.23\pm0.02$ seconds.  Checkpoints for
both methods are selected by held-out residuals without exact targets.  The
two solvers attain similar accuracy, with a shorter runtime for direct
training.

\subsection{CUDA Working Memory and Stencil Evaluation}

The scaling benchmark in \cref{smfig:cuda-scaling} uses the range-four, width-$40$ local-energy model
($3682$ parameters) and batch size $64$.  After warmup, entries are means
over ten CUDA-timed repetitions.  Peak allocated memory has log--log exponent
$0.968$ over the displayed dimensions.  The table reports complete training
steps and online policy evaluations for the structured local implementation.
The full-query comparison below times only the two stencil evaluations.

\begin{table}[tbhp]
\centering
\small
\caption{CUDA scaling of the local stencil on one NVIDIA GeForce RTX~5070.}
\label{smtab:cuda-scaling}
\begin{tabular}{rrrr}
\toprule
$d$ & peak allocated (MiB) & training step (ms) & policy evaluation (ms)\\
\midrule
64 & 123.6 & 12.62 & 3.65\\
128 & 236.7 & 12.71 & 3.60\\
256 & 456.7 & 12.68 & 3.56\\
512 & 896.5 & 12.91 & 3.60\\
1024 & 1775.9 & 19.87 & 3.85\\
2048 & 3534.8 & 38.89 & 5.77\\
\bottomrule
\end{tabular}
\end{table}

For a separate batch-$8$ comparison, the vectorized, chunked full-query
stencil is $1.91$ times faster at $d=64$ and is within $6.2\%$ of parity at
$d=128$; the local implementation is $4.04$, $15.81$, and $61.70$ times
faster at $d=256,512,1024$.  In float64 the two implementations differ by at
most $5.26\cdot10^{-14}$ in the centered gradient and
$3.45\cdot10^{-10}$ in the Laplacian (relative $5.70\cdot10^{-12}$).

\begin{figure}[tbhp]
\centering
\includegraphics[width=0.92\textwidth]{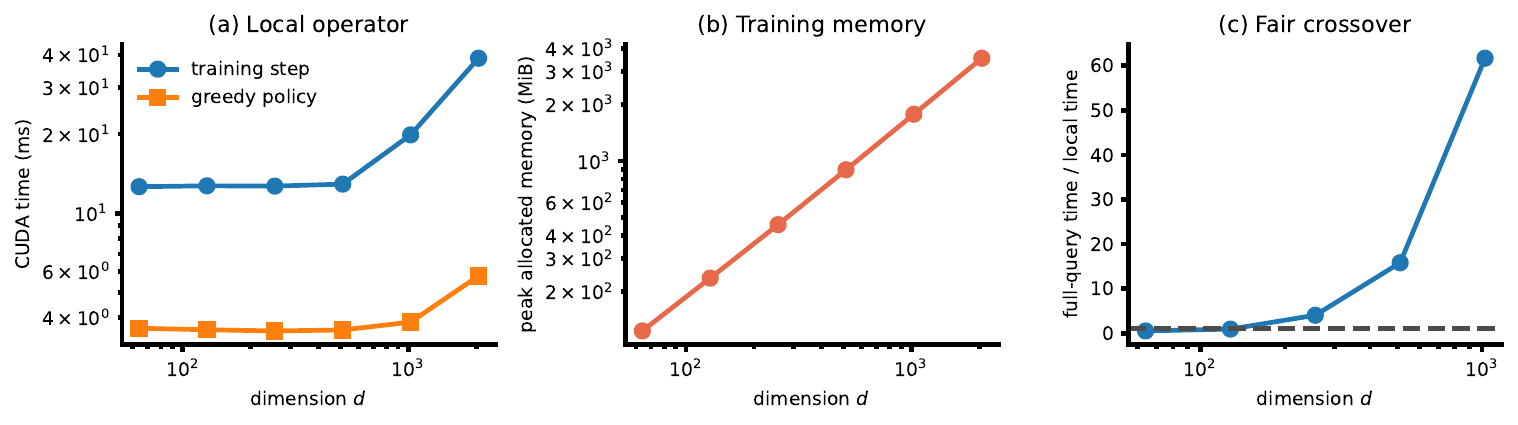}
\caption{Measured CUDA working memory, training time, policy-evaluation
time, and the full-query/local-stencil crossover.}
\label{smfig:cuda-scaling}
\end{figure}

\subsection{Sparse-Actuator Allen--Cahn Audit}

The nonmanufactured test in \cref{sec:exp-allencahn} uses
\[
 \dot x_i=2(x_{i-1}-2x_i+x_{i+1})+2(x_i-x_i^3)+m_i u_i,
 \qquad m_i=\one_{\{i\equiv1\pmod 8\}},
\]
with $L=\tfrac12|x|^2+\tfrac12|u|^2$, $T=0.75$, and $|u_i|\le3$.
The value architecture contains shared pair energies out to radius $R$.
Replicated endpoints impose $x_0=x_1$ and $x_{d+1}=x_d$, and the terminal
cost is $g(x)=|x|^2/2$.  Each initial-state ensemble mixes eight low Fourier
modes with three randomly located Gaussian pulses and is clipped to
$[-0.85,0.85]^d$.  The proportional gain is selected from
$\{0.5,1,2,3,4,5,6,8\}$ on held-out validation states before evaluation.
The ten-ensemble audit contains $64$ new initial states per ensemble and uses
the classical fourth-order Runge--Kutta method with step $0.005$ for every
policy.  Controls are held fixed within each integration step.
Halving the step changes the three mean costs by at most $0.069\%$ and the
two reported cost ratios by at most $0.016\%$.

\begin{table}[tbhp]
\centering
\small
\caption{Sparse Allen--Cahn policy-cost ratios relative to the best
proportional controller.  The $d=1024$ entries aggregate three training
seeds and ten independent test ensembles.}
\label{smtab:e10-audit}
\begin{tabular}{lcc}
\toprule
policy & dimension & cost ratio\\
\midrule
SDPI, $R=4$ & 1024 & $0.94990\pm0.00197$\\
direct nonlinear local & 1024 & $0.94372\pm0.00115$\\
SDPI, $R=0$ & 256 & $1.07780\pm0.00290$\\
SDPI, $R=1$ & 256 & $0.95855\pm0.00357$\\
SDPI, $R=2$ & 256 & $0.95423\pm0.00333$\\
SDPI, $R=4$ & 256 & $0.95370\pm0.00253$\\
\bottomrule
\end{tabular}
\end{table}

The direct nonlinear local policy is lower in cost than SDPI on every audited
state after averaging over training seeds; its mean advantage is $0.65\%$.
Accordingly, this experiment demonstrates the relevance of neighboring
unactuated states and scalable policy evaluation, not SDPI superiority.

\clearpage
\bibliographystyle{plain}
\bibliography{ref}

\end{document}